%% file: main.tex
\documentclass[11pt]{article}
\usepackage{amsmath,amsthm,amssymb}  

\usepackage{xcolor}

\usepackage{amsmath,amssymb}   

\usepackage{algorithm,algorithmic}

\usepackage[margin=1 in]{geometry}
\usepackage[bottom]{footmisc}
\usepackage[mathscr]{eucal}
\usepackage{bm}

\usepackage{xparse}

\usepackage{subcaption}
\usepackage[utf8]{inputenc} 
\usepackage[T1]{fontenc}    
\usepackage{url}            
\usepackage{booktabs}       
\usepackage{amsfonts}       
\usepackage{nicefrac}       
\usepackage{microtype}      
\usepackage{amsmath}

\usepackage{amssymb}
\usepackage{graphicx}
\usepackage{enumitem}
\usepackage{float}

\NewDocumentCommand \T { O{} m } {\boldsymbol{#1\mathscr{\MakeUppercase{#2}}}}

\NewDocumentCommand \Mx { O{} m } {{\bm{#1\mathbf{\MakeUppercase{#2}}}}} 

\NewDocumentCommand \V { O{} m } {{\bm{#1\mathbf{\MakeLowercase{#2}}}}}

\theoremstyle{plain}

\newtheorem{lemma}{Lemma}
\newtheorem{definition}{Definition}
\newtheorem{corollary}{Corollary}
\newtheorem{theorem}{Theorem}

\newtheorem{example}{Example}
\DeclareMathOperator*{\minimize}{minimize}
\DeclareMathOperator*{\rank}{rank}
\DeclareMathOperator*{\supp}{support}
\DeclareMathOperator*{\symrank}{rank_{\sf sym}}

\DeclareMathOperator*{\diag}{diag}

\DeclareMathOperator*{\vectorize}{vec}
\DeclareMathOperator*{\matricize}{Mat}
\DeclareMathOperator*{\trace}{trace}
\DeclareMathOperator*{\lmq}{\ensuremath{\mathcal{L} \left(\Mx{M},\Mx{Q}\right) }}
\DeclareMathOperator*{\llq}{\ensuremath{\mathcal{L} \left(\V{\lambda},\Mx{Q}\right)} }

\DeclareGraphicsExtensions{.eps}

\def\<{\langle}
\def\>{\rangle}

\def\R{\mathbb{R}}

\newcommand{\hide}[1]{}

\theoremstyle{plain}

\numberwithin{equation}{section}
\numberwithin{lemma}{section}
\numberwithin{theorem}{section}
\numberwithin{corollary}{section}
\numberwithin{observation}{section}
\numberwithin{definition}{section}
\numberwithin{example}{section}

\usepackage{hyperref}       
\usepackage{bookmark}       

\author{Abbas Kazemipour\thanks{Stanford University. kaazemi@stanford.edu}\and Brett Larsen \thanks{Stanford University. bwlarsen@stanford.edu}\and Shaul Druckmann\thanks{Stanford University. shauld@stanford.edu}}

\begin{document}

\title{Avoiding Spurious Local Minima in Deep Quadratic Networks}

\maketitle

\begin{abstract}
  Despite their practical success, a theoretical understanding of the loss landscape of neural networks has proven challenging due to the high-dimensional, non-convex, and highly nonlinear structure of such models. In this paper, we characterize the training landscape of the mean squared error loss for neural networks with quadratic activation functions. We prove existence of spurious local minima and saddle points which can be escaped easily with probability one when the number of neurons is greater than or equal to the input dimension and the norm of the training samples is used as a regressor. We prove that deep overparameterized neural networks with quadratic activations benefit from similar nice landscape properties.  Our theoretical results are independent of data distribution and fill the existing gap in theory for two-layer quadratic neural networks.  Finally, we  empirically demonstrate convergence to a global minimum for these problems.%
\end{abstract}

\input{QLQL_intro}
\input{QLQL_notations}

\input{QLQL_proofs}
\input{QLQL_experiments}

\input{QLQL_conc}


\section*{Acknowledgements}
B.W.L. was supported by the Department of Energy Computational Science Graduate Fellowship program (DE-FG02-97ER25308).


\bibliographystyle{unsrt}
\bibliography{main}

\newpage

\renewcommand{\theHsection}{A\arabic{section}}
\appendix

\input{app_notations}
\input{app_optim}

\input{app_proof_alpha}

\input{app_proof_multivariate}
\input{app_proof_deep}

\input{app_proof_tensor}

\end{document}

%% file: QLQL_intro.tex
\section{Introduction} \label{sec:intro}

A striking phenomenon in the application of neural networks is the success of training using simple first-order methods such as gradient descent despite the non-convex structure of the loss function.  Often these first-order methods find global minimizers of the error even when the data and labels are randomized \cite{zhang2016understanding}, leading to the natural conjecture that this behavior is due to the existence of very few or no local minima in the loss landscape.  However, due to the highly nonlinear, high-dimensional and non-convex structure of the associated optimization problems, a theoretical understanding of network training has proven challenging for virtually all architectures with non-linear activations. In particular training such networks is known to be NP-hard in general \cite{blum1989training}.

Similar observations have been made in the non-convex optimization community. Since the seminal works of Burer and Monteiro \cite{burer2003nonlinear, burer2005local} on non-convex semidefinite programming, the field has tried to characterize which problems have loss landscapes amenable to the use of first-order methods, i.e. have no spurious local minima. Recent theoretical results have proven this property for several non-convex optimization problems such as over-parameterized neural networks with quadratic activations \cite{JS19}, semidefinite matrix completion \cite{ge2016matrix, boumal2016non} and non-convex low-rank matrix recovery \cite{ge2017no,de2014global, bhojanapalli2016global}. In general, however, these results require restrictions on the problem set-up or input data generation.

In this paper we characterize the loss landscape of quadratic neural networks (QNN) and remove the restrictive constraints of previous work. In particular the main contributions of this paper can be summarized as:
\begin{enumerate}[leftmargin=0.7cm]
    \item \textbf{Data Normalization:} For a two-layer neural network with quadratic activation and $d$-dimensional input, the landscape of the quadratic loss function suffers from local minima for arbitrary number of neurons. However, when the number of hidden neurons $k$ is greater than or equal to $d$ these local minima can be easily escaped by adding the norm of the training data as a regressor. In this case with probability one over random initialization of the weights gradient descent escapes spurious local minima and saddle points with no negative curvature direction. This result emphasizes the role of data normalization, namely normalizing the samples not only suppresses outliers but also helps escape spurious stationary points. Our result closes the gap between the results of \cite{JS19} which prove this statement for $k \geq 2d$ and the lower bound of $k = d$ for full-rank data.  (This lower bound comes from the fact that if $k < d$, then we are using a rank deficient model and will always incur training error for full rank data).
    
    \item \textbf{Regularization:} Alternatively, we show that the local minima of two-layer QNNs can be escaped by regularizing the symmetric factorizations of a symmetric matrix so that only a single desired solution, namely its eigenvalue decomposition attains a zero penalty.
    
    \item \textbf{Overparameterization and Deep Quadratic Networks:} We show that deep overparameterized quadratic neural networks benefit from similar nice landscape properties, namely their landscapes do not suffer from spurious stationary points.
    
\end{enumerate}

The primary application of these results is conditions under which convergence of stochastic gradient descent to a global minimizer is guaranteed, though the structure of two-layer QNNs enables other efficient learning algorithms. In fact, training a network with the number of hidden units equal to the input dimension reduces to a linear regression (least squares) problem with quadratic features, and the weights can then be obtained via an eigendecompostion \cite{cheng2018polynomial}.  However, it is important to acknowledge that all two layer neural networks with polynomial activation functions including QNNs do not enjoy the same universal approximation theorems as two-layer neural networks with sigmoid or ReLU activations \cite{hornik91, lu2017expressive}. Indeed, the global solution to a two-layer QNN is equivalent to linear regression again quadratic features. Hence, increasing the number of neurons in a two-layer QNN beyond $k = d$ does not add any expressive power.  

In contrast, universal approximation results have been established for deep QNNs in \cite{fan18}. In particular, the global minimum of a depth-$L$ QNN is equivalent to linear regression against polynomial features of degree at most $2^L$.  The expressive power of deep polynomial networks has been considered further in \cite{kileel2019expressive} using tools from algebraic geometry.  Deep quadratic networks are of additional interest from an optimization perspective as they correspond to higher-order tensor decompositions.

Alongside the general importance of understanding neural networks, our results have specific relevance for theoretical neuroscience.  A key problem in this field is understanding responses of neurons with nonlinear behavior to high-dimensional stimuli. In general, individual neurons are typically tuned to a low-dimensional stimuli manifold, e.g. in visual space, but their responses are not a linear projection onto the manifold but rather a non-linear transformation of such a projection. Indeed, a quadratic response to a low-dimensional projection is the model typically used  for complex cells in the visual cortex \cite{rust2005spatiotemporal}. 

\textbf{Comparison to Previous Work} There is a growing body of work on theoretical guarantees for two-layer neural networks with various activation functions, including \cite{tian2017analytical, brutzkus2017globally, li2017convergence, janzamin2015beating, zhong2017recovery, panigrahy2017convergence, javanmard2019analysis, mondelli2018connection}. In \cite{javanmard2019analysis}, the authors use piecewise approximations to fit a concave input-output function $f$, specifically linear combinations of $k$ `bump-like' components (hidden neurons).  The paper demonstrate that when $k \rightarrow \infty$, the approximation error of $f$ goes to 0 at exponential convergence rates using displacement complexity.  In \cite{mondelli2018connection}, the authors leverage a comparison to tensor decompositions to prove hardness results for two-layer NN's with polynomial activation functions.  For cubics, they show that for input dimension $d$ and $k$ hidden units, there exists no polynomial-time algorithm that can outperform the trivial predictor when $d^{3/2} \ll k \ll d^2$.

In particular \cite{livni2014computational, algReg18, JS19, du2018power, venturi2019spurious} have all considered shallow networks with quadratic activations as a theoretically-tractable non-linearity and as a second-order approximation of general non-linear activations. The results of our paper for two-layer quadratic neural networks are closest to \cite{JS19} and \cite{venturi2019spurious}, and overall, learning the sign structure of the matrix associated with the network (i.e. learning the weights of the linear layer) is an important part of the generality of our work. \cite{du2018power} prove a similar landscape result that holds for positive semidefinite matrices; in contrast, our theory holds for arbitrary matrices and hence non-planted data. In \cite{JS19}, learning the sign structure is avoided by setting the number of hidden neurons to $2d$ ($d$ being the input dimension) and then using a sign structure with $d$ positive elements and $d$ negative elements.  Thus, their proof of no spurious local minima with arbitrary sign structure requires the number of hidden units be $\geq 2d$ rather than $\geq d$.  \cite{venturi2019spurious} is able to extend the results of \cite{JS19} to networks which learn the output layer weights but relaxing the results to no spurious local valleys rather spurious local minima.  None of these papers consider the extension to the multivariate output case which is important for our consideration of deep networks. These comparisons are expounded on in Section \ref{sec:previousWork}.

Networks with more general activation functions have been considered with certain regularity conditions in \cite{soudry2016no, nguyen2017loss}. The landscape of deep networks has been studied from a variety of perspectives including statistical mechanics and random matrix theory in \cite{choromanska2015loss, dauphin2014identifying, pennington2017geometry, bahri2020statistical}. Recent work on the landscape of deep networks such as \cite{kawaguchi2019depth} and \cite{kawaguchi2019effect} considers how the depth and width affect the loss landscape, but we are not aware of any other results for deep quadratic networks or quadratic networks with multivariate outputs.  

The rest of the paper is organized as follows: We introduce our notational conventions as well as the problem formulations for two layer quadratic neural networks in Section \ref{sec:notation}. Our main results and proofs are summarized Section \ref{sec:theory} and empirically verified in Section \ref{sec:experiments}. The majority of the proof detail is provided in the supplementary material.

%% file: QLQL_notations.tex
\section{Preliminaries}
\label{sec:notation}

In this section, we introduce the notation used throughout for the single-layer quadratic neural network and its loss landscape.  We review previous results for this landscape, and prove by construction that this loss landscape has spurious local minima when the output layers weights are trainable.

\subsection{Single-Layer Quadratic Linear Neural Networks}
We refer to Appendix \ref{app:notations} for a detailed list of notations. The overall input-output relationship of a two-layer neural network is a function $f: \mathbb{R}^d \rightarrow \mathbb{R}$ that maps the input vector $\V{x} \in \mathbb{R}^d$ into a scalar output in the following manner:
\begin{equation}
\V{x} \mapsto f_{\V{\lambda},\Mx{Q}}(\V{x}) := \sum_{j = 1}^k \lambda_j  \sigma \left( \V{q}_j^T {x} \right),
\end{equation}
where $\V{q_i} \in \mathbb{R}^d$ represents the weights connecting the input to the $i$-th neuron in the hidden layer, $\sigma(.)$ is the (nonlinear) activation function and $\Mx{\Lambda} = \diag(\V{\lambda}) \in \R^{k \times k}$ represents the second-layer weights. In this paper we are mostly interested in the case where the activation function is a quadratic, i.e. $\sigma(x) = x^2$. In this case
\begin{equation}
\V{x} \mapsto f_{\V{\lambda},\Mx{Q}}(\V{x}) := \sum_{j = 1}^k \lambda_j  \left( \V{q}_j^T \V{x} \right)^2= \Mx{Q \Lambda Q}^T \bullet \V{xx}^T,
\end{equation}
where we have used the notation $\Mx{Q} = [\V{q}_1,\V{q}_2, \cdots, \V{q}_k] \in \mathbb{R}^{d \times k}$ for the concatenated weights to write the network transformation succinctly.   Thus, the network has input dimension $d$, $k$ hidden neurons, and is given $n$ training examples.  Since the optimization is done jointly on the quadratic and linear weights, we refer to this architecture a \textit{single}-layer quadratic-linear (QL) network. Similarly, an $L$-layer deep QL network has $L$ quadratic layers each followed by a linear layer.

For a given training set $(\V{x}_n, y_n)_{n=1}^N$, we denote the empirical risk as:
\begin{align*}
\mathcal{L}_k \left( {\V{\lambda}, \Mx{Q}} \right) &:= \frac{1}{N}\sum_{n = 1}^N \left( y_n -  f_{\V{\lambda},\Mx{Q}}({\V{x}_n}) \right)^2 = \frac{1}{N}\sum_{n = 1}^N \left( y_n -  \sum_{j= 1}^k \lambda_j  \left( \V{q}_j^T \V{x}_n\right)^2 \right)^2 \\
&=  \frac{1}{N}\sum_{n = 1}^N \left( y_n -  \Mx{Q \Lambda Q}^T \bullet \Mx{X}_n \right)^2.
\end{align*}
where we have used the notation $\Mx{X}_n = \V{x}_n \V{x}_n^T \in \R^{d \times d}$ for the outer product of each training feature. Similarly the population loss can be defined as:
\begin{align*}
\Bar{\mathcal{L}}_k \left( {\V{\lambda}, \Mx{Q}} \right) &:= \mathbb{E}\mathcal{L}_k \left( {\V{\lambda}, \Mx{Q}} \right) = \mathbb{E} \left( y_n -  \Mx{Q \Lambda Q}^T \bullet \Mx{X}_n \right)^2.
\end{align*}

\textbf{Remark:}  Throughout the paper we will prove our results for the empirical risk. All the results of our paper however can be reproduced for the population risk by replacing the empirical averages with their population (expected) counterparts.

Consider optimizing over $\Mx{Q}$ and $\Mx{\Lambda}$. The first-order optimality conditions for $\Mx{Q}$ in (\ref{eq:mseloss_main}) are given by:
\begin{equation} 
\sum_{n=1}^N r_n \Mx{X}_n \Mx{Q \Lambda} = \Mx{0}.
\end{equation}
and the second-order optimality condition imply that for any local minimum or saddle point with no direction of negative curvature
\begin{equation} 
    \sum_{n=1}^N \left(\Mx{X}_n \Mx{Q \Lambda} \bullet \Mx{U} \right)^2 - r_n\Mx{X}_n \bullet \Mx{U \Lambda U}^T \geq 0,
\end{equation}
for all $\Mx{U} \in \R^{d \times k}$. The loss landscape is said to have no spurious local minima if there exists no minimum of the landscape that is not also a global minimum.

It is useful to define the convex unregularized counterpart of $\mathcal{L}_k$:
\begin{equation} \label{eq:main_text_cvx}
\minimize_{\Mx{A}} \frac{1}{N}\sum_{n=1}^N \left(y_n - \Mx{A} \bullet \Mx{X}_n \right)^2,
\end{equation}
where $\Mx{A}$ is the symmetric matrix defined by $\sum_{j= 1}^k \lambda_j  \V{q}_j \V{q}_j^T \in \R^{d \times d}$. Thus, the transformation performed by a single layer QL network can be completely defined by the symmetric matrix $\Mx{A}$.  The sign structure of the eigenvalues of $\Mx{A}$ will be defined by the signs of the weights $\V{\lambda}$. By convexity of (\ref{eq:main_text_cvx}) any solution satisfying the first-order optimality conditions:
\begin{equation} \label{eq:main_text_KKT_main}
\sum_{n=1}^N r_n \Mx{X}_n = 0,
\end{equation}
is globally optimal.

\subsection{Previous Results on \texorpdfstring{$\llq$}{L(l,Q)}}
\label{sec:previousWork}

The first set of relevant results consider the related landscape problem of $\mathcal{L}(\Mx{Q})$ or the single layer QNN with the output layer weights fixed.  If we assume $k \geq 2d$ and fix the weights $\V{\lambda}$ such that there are at least $d$ positive and at least $d$ negative entries, then \cite{JS19} showed that $\mathcal{L}(\Mx{Q})$ has no spurious local minima and that all saddle points have a negative direction of curvature.  Importantly, by fixing the weights in this manner the network is able to represent transformations $\Mx{A}$ with any sign structure. The same paper also consider the specific case where training data $(\V{x}_n, y_n)_{n=1}^N$ was generated via a full-rank planted model.  This means regardless of how each $\V{x}_n$ is generated the output $y_n$ is given by a quadratic neural network $\V{\lambda}^{*T} (\Mx{Q}^* \V{x})^2$ with no additive noise and for which $\sigma_{\text{min}}(\Mx{Q}^*) > 0$, and there thus exists a ground truth $(\V{\lambda}^*, \Mx{Q}^*)$ for which the loss value is 0.  Under the assumption that the sign structure of the output layer is set to match the ground truth $\V{\lambda}^*$, \cite{JS19} demonstrates that the landscape still has has no spurious local minima and that all saddle points have a negative direction of curvature for $k \geq d$.

A related result by \cite{du2018power} show the same properties hold when $k \geq d$ but when the fixed output layer has all positive weights (they specifically use a vector of all 1's) and the first layer weights are regularized by the Frobenius norm squared $\|\Mx{Q}\|_F^2$.  These restrictions change the landscape in two important ways.  First, the network can now only learn positive semi-definite $\Mx{A}$ and thus any global minimum corresponding to $\Mx{A}$ with negative eigenvalues will be eliminated.  Second, the weight regularization convexifies the landscape, changing the location of stationary points.

\cite{venturi2019spurious} is able to extend these results to the landscape $\llq$ in which the output layer weights are also trainable by relaxing the condition from no spurious local minima to no spurious local valleys.  A landscape with spurious valleys is one in which a path-connected component of a certain sub-level set does not contain a global optimum. Intuitively, the difference between these conditions is whether we consider local minima in flat regions of the landscape; a landscape with no spurious valleys can contain such local minima as moving through such points to a global optimum does not increase the loss function.  However, if these flat minima have positive measure they can still trap gradient descent trajectories.  \cite{venturi2019spurious} proved the landscape $\llq$ contains no spurious local valleys provided $k \geq 2d + 1$.

Lastly, observe that since that the transformation of an arbitrary quadratic network can be represented by the $d \times d$ symmetric matrix $\Mx{A}$, the degrees of freedom of this transformation is $n^* = d(d+1)/2$ or the number of upper diagonal elements in the matrix.  In \cite{ge2016matrix}, the authors show that if the number of samples is less than $n^*$, then a two-layer QNN can memorize it.  This condition is equivalent to saying the data is generated by a planted model; the number of samples is smaller than the degrees of freedom and thus there exists an $\Mx{A}$ that exactly reproduces the data.  Thus, this result can be viewed as a corollary of the results of \cite{JS19}.  Similarly, \cite{gamarnik2019stationary} considers the population risk of noiseless two-layer quadratic networks with planted weights.  By the degrees of freedom logic above, the population risk result can also be inferred from \cite{JS19}: if the number of samples exceeds n* for a planted model, we can recover the ground-truth transformation via training, achieving zero generalization error.

\subsection{Local Minima of \texorpdfstring{$\llq$}{L(l, Q)}}
Our first result is the existence of spurious local minima for the landscape $\llq$.  Consider the following example:

\begin{example} \label{ex:1} Let $\Mx{A} \succ \Mx{0}$ be an arbitrary positive definite matrix, $\Mx{X}_n \succ \Mx{0}$ and $y_n = \Mx{A} \bullet \Mx{X}_n$. Then any point of the form $(\V{\lambda}_0 \prec \Mx{0}, \Mx{Q}_0 = \Mx{0})$ is a local minimum of $\llq$.
\end{example}

\begin{proof}
Note that for any point $(\V{\lambda}, \Mx{Q})$ in a small enough neighborhood of $(\V{\lambda}_0, \Mx{Q}_0)$ we have $\V{\lambda} \prec \Mx{0}$. Therefore
\begin{align}
\llq = \frac{1}{N}\sum_{n=1}^N \left( \left( \Mx{A}- \underbrace{\Mx{Q \Lambda Q}^T}_{ \preceq \;\Mx{0}} \right) \bullet \Mx{X}_n \right)^2  \geq \frac{1}{N}\sum_{n=1}^N \left( \Mx{A} \bullet \Mx{X}_n \right)^2 = \mathcal{L}(\V{\lambda},\Mx{0}),
\end{align}
with equality if and only if $\Mx{Q} = \Mx{0}$.
\end{proof}

Example \ref{ex:1} is a special case of local minimum for which $\sum_{n}r_n \Mx{X}_n \succeq \Mx{0}$. As we will see, all local minima which are not global and saddle points with no negative curvature direction must satisfy $\sum_{n}r_n \Mx{X}_n \succeq \Mx{0}$ or $\sum_{n}r_n \Mx{X}_n \preceq \Mx{0}$ with probability one over random initialization. In addition, we note that in the example $\rank(\Mx{Q}) < d$. As we will demonstrate in Appendix \ref{app:prelim}, all local minima of $\llq$ will be low-rank.

Combined with the results of \cite{venturi2019spurious}, this example means that for $k \geq 2d + 1$, all minima lie in flat regions of the landscape so that spurious local minima do exist but not spurious valleys.  The existence of spurious local valleys for $d \leq k \leq 2d$ is still an open question.

%% file: QLQL_proofs.tex
\section{Theoretical Guarantees: Global Optimality of Gradient Descent} \label{sec:theory}
In this section, we state our main theoretical results and give an outline of the proofs; the full details are included can be found in the appendix. Given the existence of spurious local minima for single-layer QL networks, we propose two innovative modifications to the landscape that avoid or eliminate these spurious points.  By necessity, these change the landscape along the trajectory but are carefully designed not to change the objective value at the global minimum.

\subsection{Modified Landscape \texorpdfstring{$\mathcal{L}({\alpha}, \V{\lambda}, \Mx{Q})$}{L(a, l, Q)}}

For theoretical purposes of this paper we consider the following objective function:
\begin{align} \label{eq:L_def_main}
\mathcal{L}_k \left(\alpha, {\V{\lambda}, \Mx{Q}} \right) := \frac{1}{N}\sum_{n = 1}^N \left( y_n -  \left(\Mx{Q \Lambda Q}^T + \alpha \Mx{I} \right) \bullet \Mx{X}_n \right)^2,
\end{align}
\noindent where $\Mx{X}_n \in \mathbb{R}^{d \times d}$, $y_n \in \mathbb{R}$, $\Mx{Q} \in \mathbb{R}^{d \times k}$, $\Mx{\Lambda} = \diag(\V{\lambda}) \in \mathbb{R}^{k \times k}$ and $\alpha \in \R$. 
We wish to characterize the landscape of the associated optimization problem:
\begin{equation} \label{eq:mseloss_main}
\minimize_{\alpha, \V{\lambda}, \Mx{Q}} \mathcal{L}_k \left(\alpha, {\V{\lambda},\Mx{Q}} \right) + \gamma \|\Mx{QQ}^T - \Mx{I}\|_{\sf Fro}^2.
\end{equation}
In particular we are interested in the cases where either $\alpha = 0$ or $\gamma = 0$. For $\alpha = 0$ this objective function shows up in several non-convex optimization problems beyond shallow QNNs, including low-rank matrix completion, low-rank matrix recovery and phase retrieval.

While this modified landscape is designed to avoid or eliminate spurious local minima, it is carefully designed to not change the objective value at the global minima.  Note that for a quadratic neural network we have
\begin{equation}
    \alpha \Mx{I} \bullet \Mx{X}_n = \alpha \trace \left( \Mx{X}_n \right)= \alpha \|\V{x}_n\|^2.
\end{equation}
Therefore, optimizing over $\alpha$ simply adds the norm of the data as a fixed regressor in the network, i.e. an additional neuron that performs this fixed computation in the hidden layer but has a trainable weight in the output layer.  If desired, it can be removed by first normalizing the data.

Setting $\gamma > 0$ on the other hand penalizes the weights for not being orthogonal. We show there exists a global minima in the original landscape for which this penalty is 0 (see Lemma \ref{lem:orthogonal_penalty}) so that the objective value of the global optimizer is unchanged. In particular the regularization $\gamma$ can be set to 0 once in the basin of attraction.

\subsection{Main Theorems: Single Layer}
We now state the main theoretical results of this paper and an intuition of their proof. Detailed proofs can be found in Appendices \ref{app:prelim}-\ref{app:deep}.

\begin{theorem} [Single-Layer Quadratic-Linear Networks] \label{thm:main_text:main} [Informal version of Theorem \ref{thm:main-alpha} and Corollary \ref{thm:main-alpha-population}] For $k \geq d$ the following statements hold for gradient descent with small enough step-size: 
\begin{enumerate}
    \item For $\gamma = 0$, if $(\Mx{\Lambda},\Mx{Q}$) are initialized from an arbitrary absolutely continuous distribution, with probability 1 stochastic gradient descent converges to a global minimum of $\mathcal{L}$.
    \item For $\alpha = 0$ and $\gamma \gg \frac{1}{N} \|\V{y}\|^2$ (or $\mathbb{E} y^2$ for pouplation loss),  gradient descent will converge to a global minimum of $\mathcal{L}$.
\end{enumerate}
\end{theorem}

This result can also be extended to multivariate outputs. In the multivariate case with $M$-dimensional outputs the objective function can be expressed as:
\begin{align} \label{eq:main_text_multivariate_mseloss}
\mathcal{L} \left(\V{\alpha}, {\Mx{\Lambda},\Mx{Q}} \right) := \frac{1}{MN}\sum_{m, n = 1}^{M,N} \left(y_{mn}-\left(\Mx{Q}\Mx{\Lambda}_m \Mx{Q}^T+\alpha_m \Mx{I} \right) \bullet \Mx{X}_n \right)^2, 
\end{align}
where $y_{mn}$ is the $m$-th output of the network for the $n$-th sample. We have the following non-trivial extensions of the Theorem \ref{thm:main_text:main} to the multivariate case
\begin{theorem}  [Single-Layer QL Networks, Multivariate Output] \label{thm:main_text:main_multivariate} [Informal version of Theorem \ref{thm:main_multivariate}] For $k \geq Md$ and using random initialization from an absolutely continuous distribution, with probability 1 gradient descent with small enough step-size goes to a global  minimum of (\ref{eq:main_text_multivariate_mseloss}).
\end{theorem}

\subsection{Outline of the Proofs}
The main difficulty in proving the main theorems stems mainly from the sign structure of eigenvalues $\Lambda$ of potential local minima; for example, in \cite{JS19} the authors handled this challenge by allowing for twice as many neurons as the input dimension and fixing the signs so that all possible eigenstructures can be realized.

\begin{itemize}
\item \textbf{Step 0:} (Lemma \ref{lem:equivalence}) One of the key ideas in all of the proofs is to remove the diagonal restriction on $\Mx{\Lambda}$ by means of an equivalent optimization problem which maps a single point in the landscape $\mathcal{L}(\alpha,\V{\lambda},\Mx{Q})$ to two balls corresponding to all unitary matrices.

\item \textbf{Step 1:} By comparing the optimality conditions for $\mathcal{L}(\alpha,\Mx{\Lambda},\Mx{Q})$ to its convex counterpart (\ref{eq:main_text_cvx}) we will show any stationary point satisfying $\rank(\Mx{Q}) = d$ is a global minimum. Moreover, defining the residuals 
\[
r_n = y_n - \Mx{Q\Lambda Q}^T \bullet \Mx{X}_n
\]
the necessary and sufficient global optimality conditions are given by
\[
\sum_{n=1}^N r_n \Mx{X}_n = \Mx{0}.
\]

\item \textbf{Step 2:} We will show for any $\alpha$ spurious saddle points with no negative curvature direction and local minima are rare. In particular, we show that with probability 1, all stationary points with the same $\Mx{Q}$ and objective value (which form a hyperplane) have a negative curvature direction unless $\sum_{n=1}^N r_n \Mx{X}_n$ is semidefinite. In conjunction with the optimality conditions for $\alpha$ we conclude all such points must be global minima.
\end{itemize}

As we will see throughout the proof of Step 2, for any point $(\V{\lambda}_0,\Mx{Q}_0)$ there exists many other point $(\V{\lambda},\Mx{Q})$ which achieves the same value of $\mathcal{L}$ but has a descending direction. This observation however does not by itself preculde existence of spurious local minima or non-strict saddle points. In particular, we show the set of points in the basin of attraction of bad stationary points is of measure 0.

\begin{itemize}
\item \textbf{Step 3:} In order to prove our assertion on the regularized objective note that the penalty term $\|\Mx{QQ}^T-\Mx{I}\|_{\sf Fro}^2$ encourages the solution to be closer to a unitary matrix (and hence full-rank). Moreover, since at least for one global minimizer (given by the eigenvalue decomposition of the solution to (\ref{eq:main_text_cvx}) ) the penalty term is equal to zero, adding a penalty does not move at least one global minimum around (although it does move most of them). Based on this intuition we take the following steps. Using the descent property of gradient descent, we show that for a large enough regularizer $\gamma$ the resulting $\Mx{Q}$ should be close to a unitary matrix and full-rank.
\end{itemize}

\begin{figure}[H]
\begin{center}
\noindent
\includegraphics[width=\textwidth]{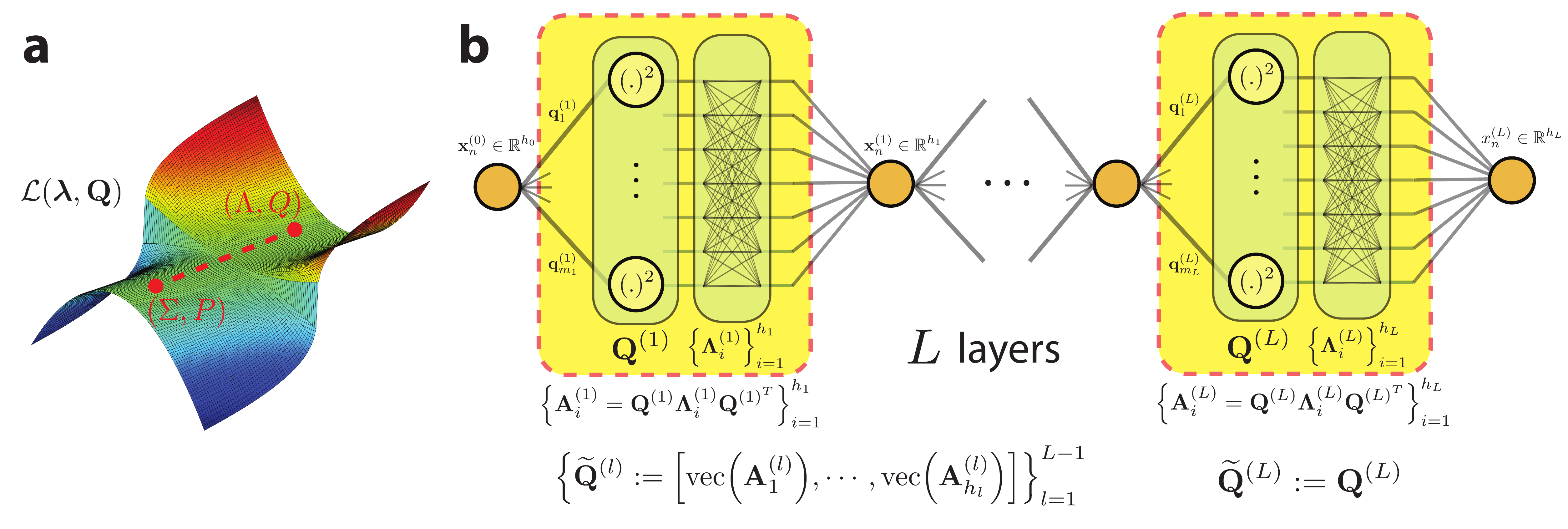}
\caption{(a) Geometric intuition into local minima in the landscape of $(\llq)$. Here two different factorizations of the same matrix are given $\Mx{A}= \Mx{Q \Lambda Q}^T = \Mx{P \Sigma P}^T$. However $\llq$ has a descent direction only for the eigenvalue decomposition $\Mx{P \Sigma P}^T$. (b) A deep quadratic-linear (QL) neural network and the related variables}
\label{fig:deep_notation}
\end{center}
\end{figure}

The main reason for the existence of local minima in the landscape of $\llq$ is the fact that the set of matrix factorizations of the form $\Mx{A}= \Mx{Q\Lambda Q}^T$ is not a singleton and limited to an eigenvalue decomposition (Example \ref{ex:1} is such an example). However if we restrict the optimization to orthonormal $\Mx{Q}$'s such a descent direction will always be available. Namely, suppose $(\V{\lambda},\Mx{Q})$ is a spurious local minimum or saddle-point and let $\Mx{A} = \Mx{Q \Lambda Q}^T$ have an eigenvalue decomposition of the form $\Mx{P\Sigma P}^T$. Since $\Mx{A}$ is not a global minimum to (\ref{eq:main_text_cvx}) there exists a descent direction $\Mx{\Delta}$, i.e. such that $\mathcal{L}_{\sf cvx} (\Mx{A + \Delta})<\mathcal{L}_{\sf cvx} (\Mx{A})$. Then by continuity of $(\V{\sigma},\Mx{P})$ in $\Mx{A}$ we conclude that there must exist a descent direction in the $(\V{\sigma},\Mx{P})$. In fact, the regularizer is a forcing mechanism which ensures $\Mx{Q}$ stays (close to) a unitary matrix throughout the optimization. Figure \ref{fig:deep_notation}a provides a geometrical intuition on the possibility of local minima in the landscape of $\llq$.

\subsection{Main Theorem: Deep Networks} \label{sec:deep}
In this section we extend our theoretical guarantees to deep quadratic-linear networks. Figure \ref{fig:deep_notation}b shows the structure of the proposed network. Note that each quadratic layer is followed by a linear layer (hence the name Quadratic-Linear networks). The addition of the extra linear layers is an important difference in QL-networks compared to regular quadratic networks where only the output layer is linear. The addition of the linear layers comes at a small cost and has great advantages. The linear weights can be thought of coefficient of the learned quadratic \textit{dictionaries} in the preceding layer and provide diversity in the sign of the eigenvalues of the resulting input-output matrix per layer. Such a diversity allows for a richer class of functions to be represented by each layer and is key to avoiding spurious local minima in our network. For an unregularized network linear layer weights (except for the last layer) can be absorbed into the quadratic weights of the subsequent QL layer. However, the design of our penalty makes explicit use of these weights, hence not allowing them to be factored in.

From a representation point of view, all functions with multivariate Taylor expansions can be represented arbitrary closely using deep QL-networks. In particular such expansions exist for analytic functions. Therefore, understanding such networks could provide useful intuition into landscape of more general architectures. We will next state the main theoretical guarantees of our paper on the landscape of deep QL-networks. We denote the parameters of the network by $(\lambda,Q)$. Detailed notation is provided in Appendix \ref{app:deep}.

\begin{theorem}[Deep Quadratic-Linear Networks] \label{thm:main_text:main_deep} [Informal version of Theorem \ref{thm:main_deep}] For an overparameterized deep network of arbitrary depth with quadratic activation functions, all critical points of penalized MSE loss given by
\[
\minimize_{(\lambda,Q)} \frac{1}{N}\sum_{n=1}^N \left( y_n - x_n^{(L)}\right)^2
\]
penalized with a layer-wise penalty similar to (\ref{eq:mseloss_main}) all global minima.
\end{theorem}
\noindent \textbf{Sketch of the proof:} The proof of Theorem \ref{thm:main_text:main_deep} consists of two major steps. The first step involves an induction on the depth of the network ($L$), where each QL pair increases the depth by 1. The base case of depth-1 reduces to Theorem \ref{thm:main_text:main}, Statement 2. Given enough neurons, a composition of two depth-1 networks $(L=2)$ can be re-written as a depth-1 network. In particular, the case $L=2$ is a regression against quartic features ($\V{x}_n^{\otimes 4}$) which can  interpreted as a quadratic regression against quadratic features ($\Mx{X}_n^{\otimes 2}$).  In this manner, a network of depth-$L$ can be rewritten as a network of depth-1 followed by a network of depth-$(L-1)$ acting as the \textit{effective} quadratic and linear layers respectively, and thus the logic of Theorem \ref{thm:main_text:main} can be used.

The second step involves ensuring that the hidden layers which we unfold during the inductive process do not get stuck at local minima. This is taken care of via overparameterization of the preceding layers such that the resulting input-output transformation matrix for the equivalent quadratic network is full-rank. Therefore the logic of Theorem \ref{thm:main_text:main}, Statement 2 can again be applied under reasonable initialization.

%% file: QLQL_experiments.tex
\section{Experiments} \label{sec:experiments}

In this section, we empirically verify our results that first-order optimization methods find a global minimum for quadratic networks in the setups of Theorems \ref{thm:main_text:main} and \ref{thm:main_text:main_deep}.  For the first set of experiments, we train QNN's with one and two quadratic-linear layers on several types of synthetic data sets. Importantly, to test that our results do not require distributional assumptions we train on both planted data generated from a random version of the network and non-planted data where the inputs and outputs are generated independently. We then turn to binary classification problems derived from the MNIST data set \cite{lecun2010mnist}.  These experiments show that the theory holds on real data sets and demonstrates that quadratic networks with two QL layers are already very expressive.\footnote{The code for all experiments is available at \url{www.github.com/druckmann-lab/QuadraticNets}}

Each theorem specifies a model, a loss function, and initialization method.  Under these specifications, the experiments train at least 50 independent realizations of the model and then looking at the final residual of the network on the task. This residual is normalized by the squared $\ell_2$-norm of the output values such that a network which simply outputted 0 for all inputs would have a residual of 1.  This error is then averaged across realizations of the network trained on the same data and reported as the ``Average Normalized Error.'' The models are trained for for at least 20,000 epochs using the Pytorch implementation of ADAM with a learning rate of $1\cdot 10^{-3}$ \cite{paszke2017automatic, kingma2014adam}.

To determine if each neural network was successfully optimized, the global minimum of the objective $\llq$ was found via linear regression (least squares) on higher-order features given by the convex optimization problem (\ref{eq:main_text_cvx}).  Across multiple experiments, the ``Fraction achieving Global Minimizer'' refers to the fraction of experiments which found a solution with a loss value equal to that of the global minimizer.  The vertical dotted line on all plots indicates the theoretical value above which all networks should achieve the global minimizer as there are no spurious local minima.

\subsection{Experiments Related to Theorem \ref{thm:main_text:main}} 

Theorem \ref{thm:main_text:main} gives several versions of the single QL layer quadratic network for which there are no spurious local minima once $k$, the number of hidden units in the quadratic layer, is greater than the input dimension of the data $d$. To test this theorem, we conducted several experiments on simulated data, each of which we train the network on data with input dimension $d = 10$ and 1500 training samples. While holding the input dimension fixed, we varied the number of hidden neurons in the quadratic layer in the range $k = 0-20$.  For each data set, confirmation of the theory is indicated by all models with $k \geq 10$ reaching the global minima; this point is marked by a vertical dotted line on the plots. Table \ref{tab:main} provides the details of each experiment.

As specified in the table, we train two variants of the QNN for which the theory holds. The first we call ``Added Norm ($\gamma = 0$)," which adds a fixed neuron that outputs the norm of the data and has weight $\alpha$.  For this network, the initialization is random and the orthogonality penalty is turned off by setting $\gamma = 0$.  The second we call ``Orthogonality Penalty ($\alpha = 0$)."  Here the regularization parameter for the orthogonality penalty $\gamma$ is set to the squared norm of the output data $\sum_{n=1}^N y_n^2$ while the norm neuron is turned off by setting $\alpha$ to 0.  Furthermore, the theorem specifies a particular initialization of the weights.  Lastly, for comparison we trained a regular quadratic network in which both $\alpha$ and $\gamma$ are 0.  By Example \ref{ex:1}, we know this network has local minima and \textit{does not} satisfy the conditions of the theorem.

The inputs to the model were always drawn i.i.d. Gaussian ($\V{x}_n \overset{\text{\sf i.i.d}}{\sim} \mathcal{N}(0, \Mx{I}_d))$ except for the first dimension which was fixed to be 1. The outputs were generated three different ways: (1) from the planted model $y_n = \Mx{A}\bullet \Mx{X}_n$ where $\Mx{A}$ is diagonal with random signs, (2) from the planted model $y_n = \Mx{A}\bullet \Mx{X}_n$ where the elements of $\Mx{A}$ are i.i.d. Gaussian, and (3) from a non-planted model where $y_n$'s are i.i.d. Gaussian independent of the inputs.  In each instance we used $N = 1500$ samples. 

\begin{table}
\caption{Experiments Related to Theorem \ref{thm:main_text:main}}
\label{tab:main} 
\begin{tabular}{ccccc}
\toprule
Figure & Output $y_n$ generation model &   $\gamma$ & $\alpha$ & Initialization \\
\midrule
\ref{fig:scalar_threshold}a & 
$\begin{array}{c}
y_n = \Mx{A}\bullet \Mx{X}_n \\ 
\Mx{A} \text{ random diagonal}  \pm1     
\end{array}$ &        
$\begin{array}{c}
     0  \\
     \frac{1}{N}\sum_{n=1}^N y_n^2
\end{array} $                         &   $\begin{array}{c}
      \text{Train}   \\
     0
\end{array} $    & $\begin{array}{c}
      \text{Random}   \\
     (\V{\lambda}_0,\Mx{Q}_0) = (\Mx{0},\Mx{I}) 
\end{array} $ \\
\midrule
\ref{fig:scalar_threshold}b & $y_n = \Mx{A}\bullet \Mx{X}_n$; $\Mx{A}$ i.i.d. $\mathcal{N}(0,1)$     &        $\begin{array}{c}
     0  \\
     \frac{1}{N}\sum_{n=1}^N y_n^2
\end{array} $                         &   $\begin{array}{c}
      \text{Train}   \\
     0
\end{array} $    & $\begin{array}{c}
      \text{Random}   \\
     (\V{\lambda}_0,\Mx{Q}_0) = (\Mx{0},\Mx{I}) 
\end{array} $ \\
\midrule
\ref{fig:scalar_threshold}c & $y_n \overset{\text{\sf i.i.d}}{\sim}\mathcal{N}(0,1)$ independent of $\V{x}_n$     &        $\begin{array}{c}
     0  \\
     \frac{1}{N}\sum_{n=1}^N y_n^2
\end{array} $                         &   $\begin{array}{c}
      \text{Train}   \\
     0
\end{array} $    & $\begin{array}{c}
      \text{Random}   \\
     (\V{\lambda}_0,\Mx{Q}_0) = (\Mx{0},\Mx{I}) 
\end{array} $\\
\bottomrule
\end{tabular}\\
\end{table}   

\begin{figure}
\begin{center}
\noindent
 \includegraphics[width=0.8\textwidth]{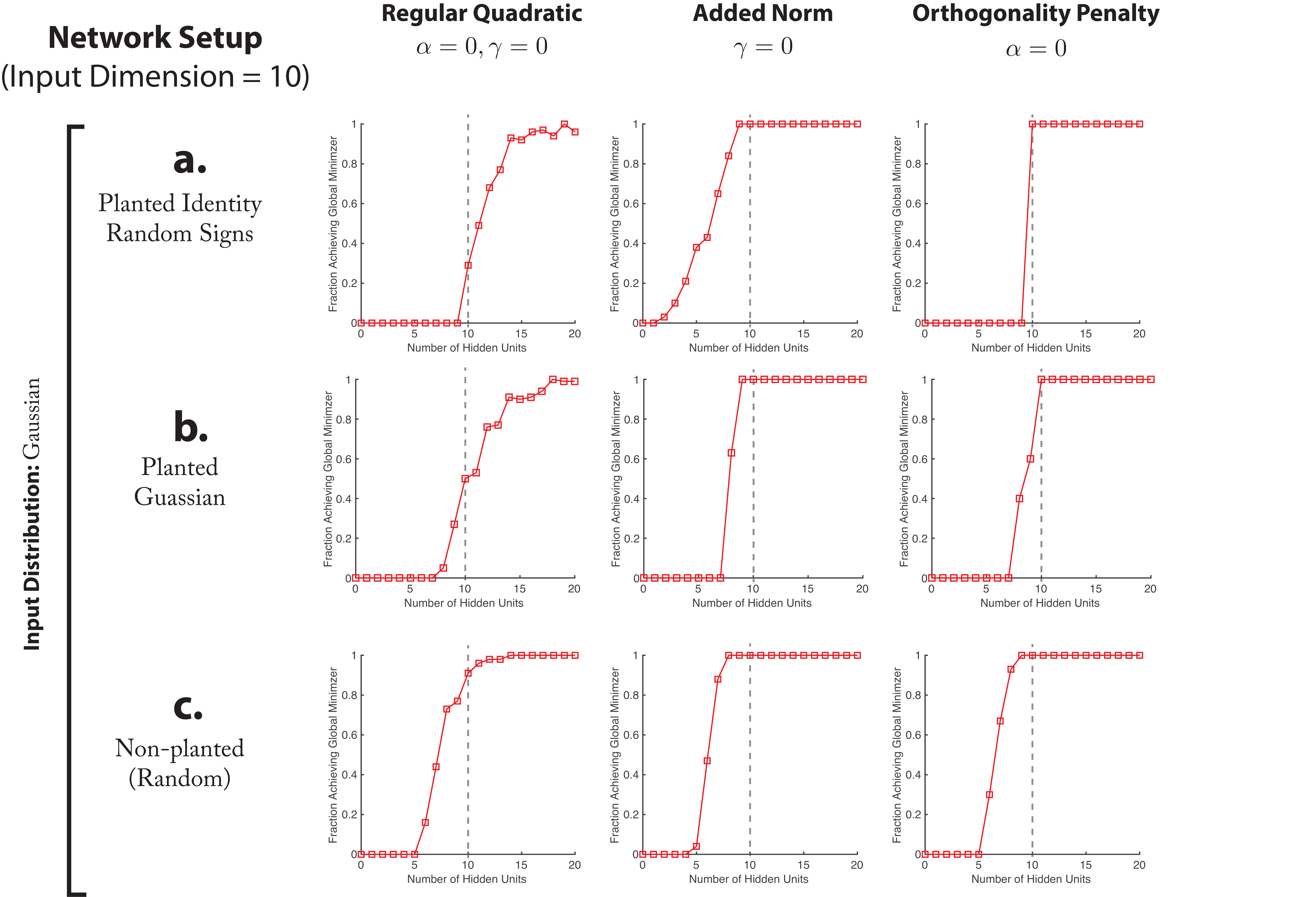}
\caption{For each experiment, we show the fraction of trials achieving the global minima for the network setup in Table \ref{tab:main}.  The average is taken over 100 trials (20 runs each for 5 random data sets). The dotted line on each plot shows the the theoretical prediction for which a network with the number of hidden units greater than or equal to this value will have no spurious local minima.}
\label{fig:scalar_threshold}
\end{center}
\end{figure}

For each network and number of hidden units, we performed a total of 100 trials (20 runs each for 5 random data sets).  Each network was trained for 30,000 epochs of ADAM, and we reported two values: The normalized mean-squared error defined by NMSE = $\sum_{n=1}^N r_n^2 / \sum_{n=1}^N y_n^2$ and the fraction of trials achieving global minima determined by a NMSE value withing $0.005$ of the solution found by least squares. The threshold values are shown in Figures \ref{fig:scalar_threshold}a, \ref{fig:scalar_threshold}b and \ref{fig:scalar_threshold}c.  Figure \ref{fig:scalar_average} shows the average NMSE broken out by data block for the planted experiments (20 runs per data block).  Given the model is planted, we know that the global optimum has an NMSE of 0.

\begin{figure}
\begin{center}
\noindent
\includegraphics[width=\textwidth]{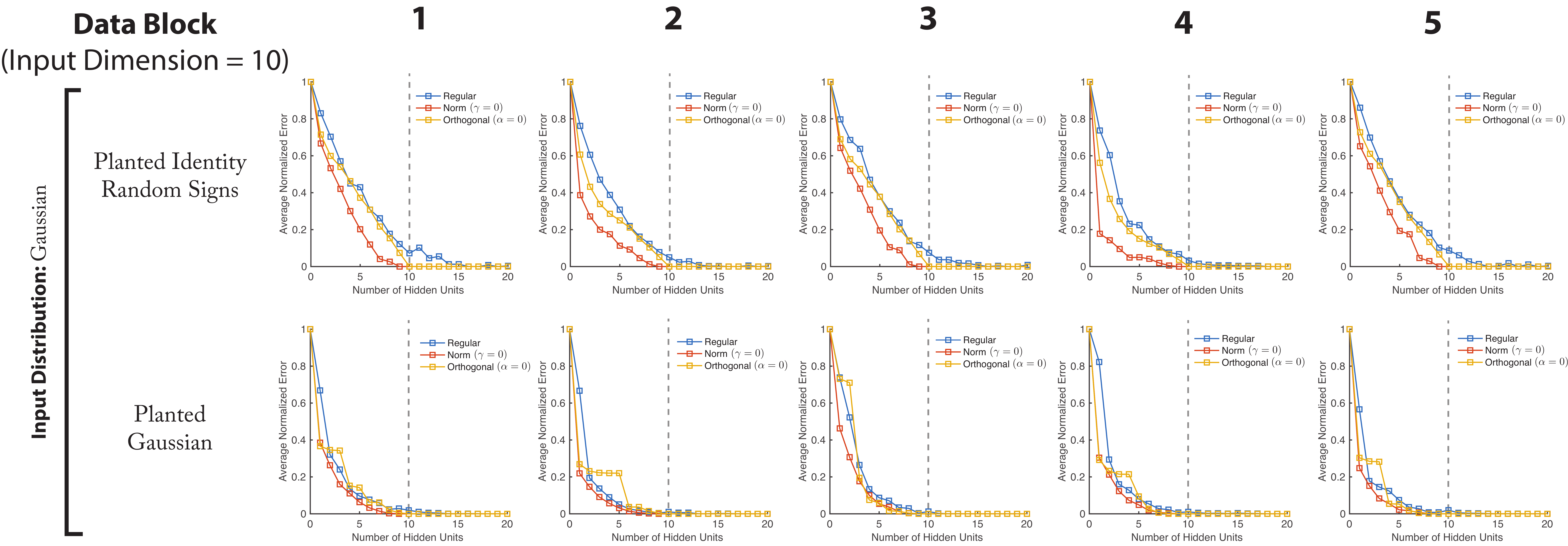}
\caption{For the experiments in Figure \ref{fig:scalar_threshold}, we show the average normalized error for each data block and network setup.  The dotted line on each plot shows the the theoretical prediction for the number of hidden units above which our theory guarantees no spurious local minima.}
\label{fig:scalar_average}
\end{center}
\end{figure}

The simulations are consistent with Theorem \ref{thm:main_text:main} in the sense that global minima are achieved when the number of hidden layer neurons is equal or greater to the input dimension ($k \geq d$) for the Added Norm and Orthogonality Penalty models.  Furthermore, they confirmed the existence of spurious local minima in networks with no modifications which are eliminated by the other set-ups.

\subsection{Experiments for Theorem \ref{thm:main_text:main_deep}: Synthetic Data}

We next test our Theory on deep QL networks as described in Theorem \ref{thm:main_text:main_deep} on synthetic data. Here we consider two quadratic-linear layers for which the hidden units before the quadratic activation is set to meet the conditions of the theorem, and the number of hidden units between the two QL layers $h_1$ is varied.  (Importantly, this parameter is different than the one we varied in the single QL layer case.  There we varied the number of hidden units in between the quadratic and linear layer; here we vary the number of units in between the two QL layers).

In these experiments the network was only trained with the orthogonality penalty ($\alpha = 0$) and for 20,000 epochs of ADAM. For each type of data generation, we performed 50 trials (10 runs each for 5 random data sets), of which three of the data sets are shown in Figure \ref{fig:twoLayerExp}.  For the planted Gaussian tensor model, the global optima has an NMSE of 0; for the independent data, the optimal value was found via least squares and is denoted by a horizontal dotted line on the plot.  The threshold plot is for all 50 runs and shows the network reaches its global minimum for $h_1 \geq d^2 = 81$.  The experiments reach the global optima for $h_1$ much less than this value which our theory does not explain; however, this may be related to the typical rank of a symmetric tensor.

\begin{figure}
\begin{center}
\noindent
\includegraphics[width=\textwidth]{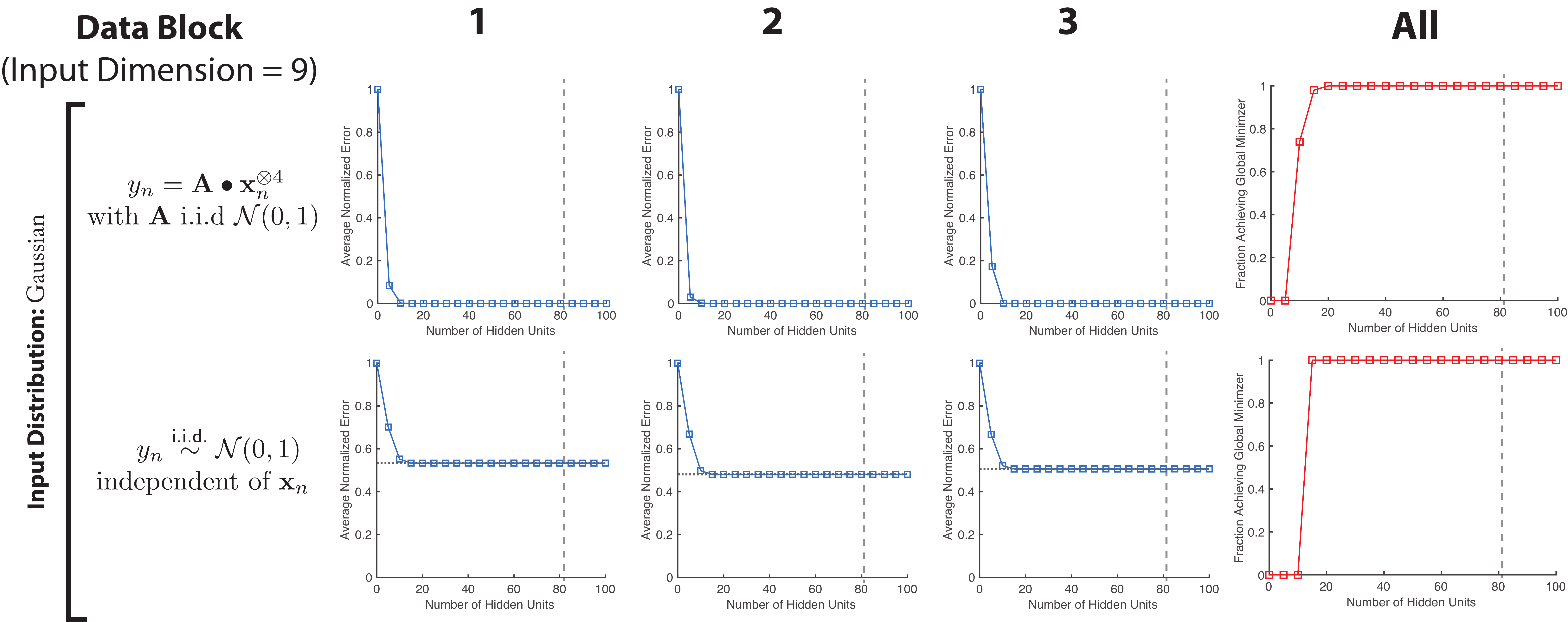}
\caption{Two-layer QL network performance on planted and independent input-output data. The value of the true global minimum is determined by least squares for non-planted data.}
\label{fig:twoLayerExp}
\end{center}
\end{figure}

\subsection{Experiments for Theorem \ref{thm:main_text:main_deep}: MNIST}

Finally, we test our theory on deep QL networks on a real data task constructed from the MNIST data set.  Here we formulate binary classification as a regression problem in the following manner.  First, we take two pairs of similar digits, 4 vs. 7 and 3 vs. 8, and divide them randomly into a training set, using approximately 1/7 of the examples and use the remaining examples as a test set.  This means for the 3 vs. 8 task we have 1539 training samples and 12,427 test examples; for the 4 vs. 7 task, we have 1504 training samples and 12,613 test examples.  We then find the first 10 principle components on the training set only and calculate the projection onto these components for both the training and test set.  Lastly, we normalize each example to have a norm of 1 and concatenate a first dimension with a fixed value of 1 (in our model this serves the purpose of adding a bias to the network).

\begin{figure}
\begin{center}
\noindent
\includegraphics[width=0.75\textwidth]{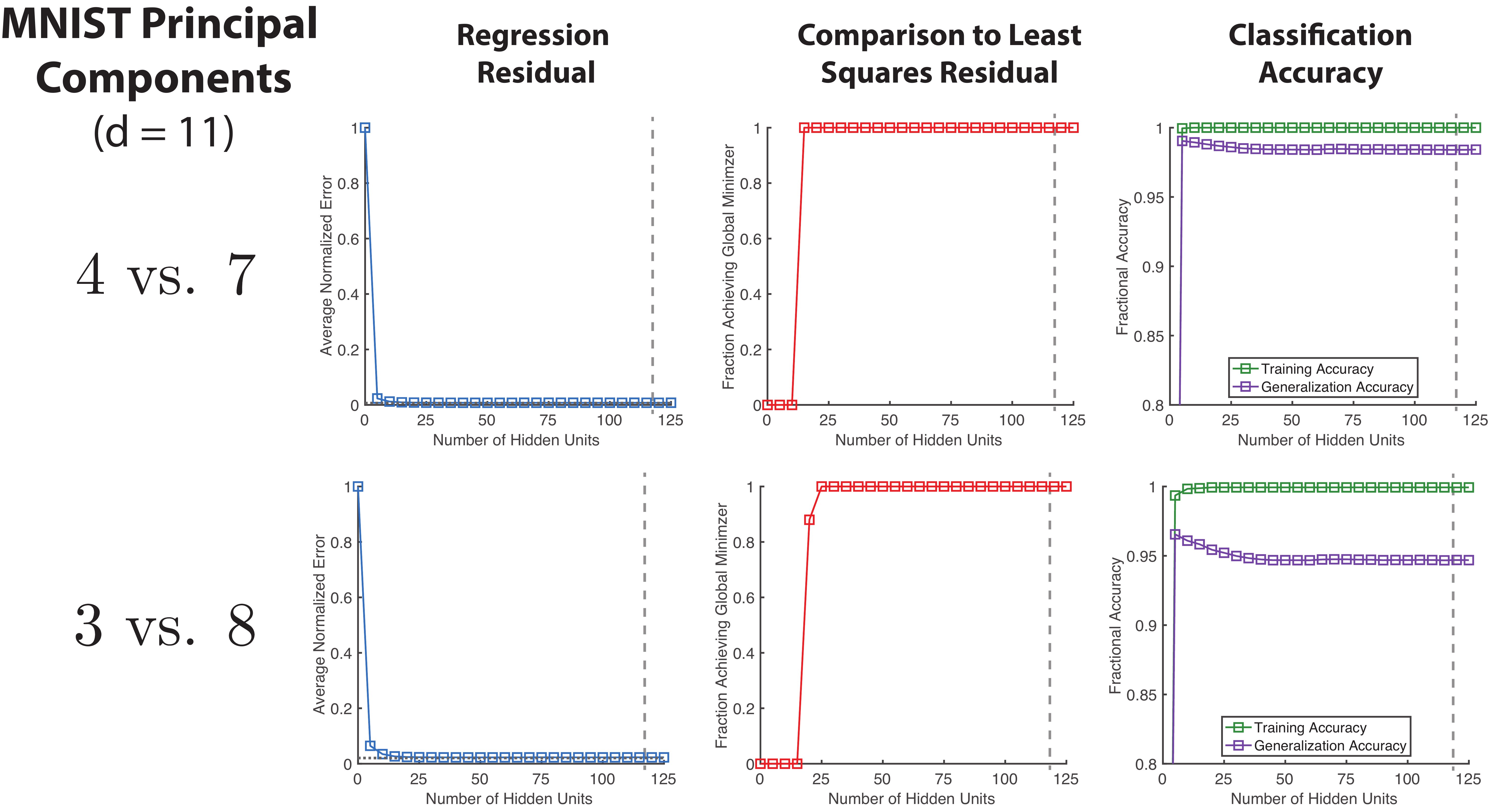}
\caption{Experiments on the binary classification tasks constructed from the MNIST data set.  The networks are trained on a regression task in which examples are labelled $\pm 1$ based on their class and then evaluated for accuracy by comparing the sign of the network output and the sign of the class.}
\label{fig:MNIST_Exp}
\end{center}
\end{figure}

For each example, we assign the value of $y_n$ to be $+1$ if it belongs to one class and $-1$ if it belongs to the other.  Thus, we train the network as a regression task using the MSE between the network output and $y_n$, but we can evaluate the classification accuracy of the network by checking whether the sign of the output matches the class of the example.

Figure \ref{fig:MNIST_Exp} show the results of this experiment over 50 realizations of the network with the orthogonality penalty.  In the first plot, we show the average NMSE across the realizations with a horizontal dotted line indicating the value of the global optima obtained by least squares (note this value is very close to 0 for the 4 vs. 7 task).  In the second plot, we show the fraction of the networks that obtained the same residual value as the global optimum, again viewed as a regression task.  Finally we look at the test and generalization accuracy of the networks.  On all plots, the vertical dotted line at $h_1 = 11^2 = 121$ indicates the theoretical value above which we predict there should be no spurious local minima.

These experiments confirm our theoretical predictions, and again achieve the global optima at a value of $h_1$ much less than $d^2$ which our theory does not explain. In addition, they show already at QL layers, QNN have sufficient expressive power for basic classification.

%% file: QLQL_conc.tex
\section{Conclusions and Future Work}

In this paper, we characterize the landscape of the mean-squared error training loss for quadratic shallow and deep neural networks. In particular, we showed that spurious local minima and saddle points of the landscapes can be easily escape via regularization or data augmentation when the number of neurons is greater than or equal to the input dimension. Our results are optimal in the sense that the number of neurons to achieve global optimality cannot be reduced in general, hence filling the existing gap in theory for these problems.

Our results do not make any statements on the uniqueness of the solution; in particular if the global minimum is not unique, Theorem \ref{thm:main_text:main} suggests convergence of gradient descent to one of the solutions but not a particular or structured one. This has been resolved in the case where one of the solutions is a semidefinite low-rank matrix (i.e. the case where $\lambda_i = 1$) \cite{ge2017no, algReg18, bhojanapalli2016global} where it has been observed that if the measurements satisfy RIP-$2r$ then given a small orthogonal initialization, gradient descent implicitly regularizes for the nuclear norm and converges to a low-rank solution. However, as we showed in Example \ref{ex:1} under arbitrary eigenvalue sign structures there exist spurious local minima which cannot be escaped without regularization. Therefore it is in a sense futile to look for implicit regularizations in the general overparameterized case if the eigenvalue structure is not known a priori.

The main remaining open question of this paper and the associated line of research is how perform a similar landscape analysis on network architectures of more practical interest. These include other non-linearities like tanh and ReLU as well as Convolutional Neural Networks.  A possible approach via polynomial networks is to consider a higher-order Taylor expansion of the activation function, but how best implement this idea remains an open problem.  More likely such networks will require new approaches to the landscape analysis.

%% file: app_notations.tex
\section{Notations} \label{app:notations}
\renewcommand{\theequation}{\Alph{section}.\arabic{equation}}
\renewcommand{\thefigure}{\Alph{section}.\arabic{figure}}

Throughout the appendix we use the following notational conventions: 
\begin{itemize}
    \item Scalars, vectors and matrices/tensors are represented by lower-case letters, bold lower-case letters and bold upper-case letters respectively (e.g. $\lambda \in \R$, $\V{q} \in \R^d$ and $\Mx{Q}\in \R^{d_1 \times d_2}$).
    \item For a matrix $\Mx{Q}$ we denote its $i$th column by $\V{q}_i$.
    \item By $\V{\lambda} \succeq 0$ we mean elementwise nonnegativity and by $\Mx{A} \succeq 0$ we mean $\Mx{A}$ is semi-positive definite.
    \item For a vector $\V{\lambda} \in \R^d$ we denote its $j$-th element by $\lambda_j$ and the associated square diagonal matrix by $\Mx{\Lambda} = \diag(\V{\lambda})$ where $\Lambda_{ii} = \lambda_i$. 
    \item For a linear transformation $\Mx{A}\in \R^{d \times d}$ and $S \subset \R^d$ we define $\Mx{A}S:= \{\Mx{A}\V{x}|\V{x} \in S\}$.
    \item For two matrices/tensors $\Mx{A}$ and $\Mx{B}$  we denote their inner product by $\Mx{A} \bullet \Mx{B} = \trace(\Mx{A}^T\Mx{B})$, their Hadamard (elementwise) product by $\Mx{A} \odot \Mx{B}$ and their outer product by $\Mx{A} \otimes \Mx{B}$. We use the notation $\underbrace{\Mx{A} \odot \Mx{A} \odot \cdots \odot \Mx{A}}_{p \text{ times }} = \Mx{A}^{\odot p}$ and $\underbrace{\Mx{A} \otimes \Mx{A} \otimes \cdots \otimes \Mx{A}}_{p \text{ times }} = \Mx{A}^{\otimes p}$.
    \item For a matrix/tensor $\Mx{A} \in \R^{m_1 \times m_2 \times \cdots \times m_k}$ by $\V{a} = \vectorize(\Mx{A})$ we mean a vector $\V{a} \in \R^{m_1m_2\cdots m_k}$ which is the vectorized form of $\Mx{A}$. With a slight abuse of notation we use the notation $\Mx{A}\bullet \Mx{B} = \vectorize(\Mx{A}) ^T \vectorize(\Mx{B})$ for \textit{any} two symmetric tensors of with the same number of elements.

    \item We denote the $r \times r$ identity matrix by $\Mx{I}_r$ and the $m \times n$ all zero matrix by $\Mx{0}_{m \times n}$.
    \item For a measurable set $S \subset \R^d$, we denote its Lebesgue measure by $\mu(S)$.
    \item For a matrix $\Mx{A}$ we denote its Frobenius and spectral norms by $\| \Mx{A} \|_{\sf Fro}$,$\| \Mx{A} \|$ respectively.
    \item For a matrix $\Mx{A}$ we denote its largest and smallest magnitude eigenvalues by $\lambda_{\max}(\Mx{A})$ and $\lambda_{\min}(\Mx{A})$ respectively.

\end{itemize}
\begin{definition} \label{def:equivalence} Two function $f(\V{x})$ and $g(\V{y})$ are said to have equivalent landscapes if for any stationary point $\V{x}^\star$ of $f$ there exists a stationary point $\V{y}^\star$ of $g$ (and vice versa) such that  $\V{x}^\star$ and $\V{y}^\star$ are both local minima, local maxima, saddle points with no negative curvature direction, or saddle points with a negative curvature direction. We denote equivalence of landscapes by
\[
f(\V{x}) \equiv g(\V{y}).
\]
\end{definition}
Note that in the definition of equivalence we have not assumed the correspondence to be one to one. In particular, adding a dimension to a function results in an equivalent landscape.

\section{Basic Optimization and Algebraic Results} \label{app:basic_opt}
The following Lemmas are basic results in optimization and algebra, some of which we state without proof.

{\begin{lemma} \label{lem:restriction}
Let $(\V{x}_0,\V{y}_0)$ be a local minimum to $f(\V{x},\V{y})$. Then $\V{x}_0$ is a local minimum to $f(\V{x},\V{y}_0)$.
\end{lemma}}

\begin{lemma} [Linear Transformations Preserve Landscapes] \label{lem:lin_transform} Let $f(x)$ be a differentiable function in $\V{x} \in \R^d$, $\V{b} \in \R^d$ and $\Mx{M} \in \R^{d \times d}$ an invertible matrix. Then
\[
f(\Mx{M}\V{x}+\V{b}) \equiv f(\V{x}).
\]
\end{lemma}

\begin{lemma} Let $f(x) = a_n x^n + a_{n-1}x^{n-1}+\cdots+a_0$ be a polynomial. Then the roots of $f(x)$ are continuous in $(a_n,a_{n-1},\cdots,a_0)$.
\end{lemma}
\begin{corollary} \label{cor:continuous_eig}
Let $\Mx{A} \in \R^{d \times d}$ be a symmetric matrix and let $(\V{\sigma}(\Mx{A}),\Mx{P}(\Mx{A}))$ be its eigenvalues and eigenvectors respectively. Then $\V{\sigma}(\Mx{A})$ and $\Mx{P}(\Mx{A})$ are continuous functions in $\Mx{A}$.
\end{corollary}

\noindent Throughout we work with an orthogonality penalty on the weights of the network.  Here we derive properties of this penalty:

\begin{lemma} \label{lem:qq_lb} Let $\rank(\Mx{Q}) = r$. Then
\[
\left\|\Mx{QQ}^T - \Mx{I}_{d \times d} \right\|_{\sf Fro}^2 \geq d-r.
\]
\end{lemma}
\begin{proof}
Let $\Mx{QQ}^T$ have an eigenvalue decomposition of the form $\Mx{P \Sigma P}^T$. Then
\[
\left\|\Mx{QQ}^T - \Mx{I} \right\|_{\sf Fro}^2 = \left\|\Mx{P \Sigma P}^T - \Mx{PP}^T \right\|_{\sf Fro}^2 = \left\|\Mx{\Sigma}- \Mx{I}\right\|_{\sf Fro}^2 \geq d-r,
\]
where the last inequality is due to the fact that the diagonal of $\Mx{\Sigma}$ has at least $d-r$ zeros.
\end{proof}

\begin{lemma}[Properties of the orthogonality penalty] \label{lem:orthogonal_penalty} Any stationary point of gradient descent fot the problem
\[
\minimize_{\Mx{Q} \in \R^{d\times d}} \left\|\Mx{QQ}^T - \Mx{I} \right\|_{\sf Fro}^2,
\]
either has a negative curvature direction or is a global minimum (i.e. satisfies $\Mx{QQ}^T = \Mx{I}$).
\end{lemma}
\begin{proof}
The second order optimality condition implies that for any $\Mx{V} \in \R^{d\times d} $ we have

\begin{align}
    \Mx{\nabla}^2_{\Mx{Q}}\Big(\left\|\Mx{QQ}^T - \Mx{I} \right\|_{\sf Fro}^2 \Big) \bullet \Big( \Mx{V \otimes V} \Big) = \left\|\Mx{QV}^T \right\|_{
    \sf Fro}^2+ 2\left(\Mx{Q}\bullet \Mx{V}\right)^2 - \|\Mx{V}\|_{\sf Fro}^2 \geq 0. 
\end{align}
If $\rank(\Mx{Q}) < d $ then it has a nonzero vector $\V{v}$ in its null space. Taking $\Mx{V} = \V{uv}^T$ for arbitrary $\V{u}$ results in $\Mx{QV}^T = \Mx{0}$ which sets the first two terms to zero and prove existence of a negative curvature direction. 
If $\rank(\Mx{Q}) = d$ the first order optimality condition yields
\[\nabla_{\Mx{Q}} \left\|\Mx{QQ}^T - \Mx{I} \right\|_{\sf Fro}^2= 2(\Mx{QQ}^T-\Mx{I})\Mx{Q}=\Mx{0},
\]
from which we conclude $\Mx{QQ}^T = \Mx{I}$.
\end{proof}

\begin{lemma} \label{lem:lin_det}
Let $\left\{\Mx{A}_i\right\}_{i = 1}^\infty \in \R^{d \times d}$ be an arbitrary sequence of invertible linear transformation, $S_1, S_2 \subset \R^{d}$ be two arbitrary subset of $\R^d$ and $\mu(.)$  be the Lebesgue measure. Then
\[
\frac{\mu(S_1)}{\mu(S_2)} = \frac{\mu(\Mx{A}_k\Mx{A}_{k-1} \cdots \Mx{A}_1 S_1)}{\mu(\Mx{A}_k\Mx{A}_{k-1} \cdots \Mx{A}_1S_2)},
\]
for all $k$.
\end{lemma}
\begin{proof} The claim is a simple consequence of the fact that 
\[
\mu(\Mx{A}S)= |\det(\Mx{A})|\mu(S).
\] 
which we apply repeatedly to the sequence of invertible linear transforms.\end{proof}

\begin{lemma}\label{lem:basin} Consider an optimization problem
\begin{equation} \label{eq:theta_x}
    \minimize_{\V{\theta},\V{x}} \| \V{y}-\Mx{A}_{\V{\theta}}\V{x}\|_2^2,
\end{equation}
such that $\Mx{A}_{\V{\theta}}$ is measurable in $\V{\theta}$ and $\eta \|\Mx{A}_{\V{\theta}}\|^2 \leq 1$ for all $\V{\theta}$. Let $(\V{\theta}^\star,\V{x}^\star)$ be a stationary point of gradient descent with step size $\eta$, $S = \{\V{x}|\Mx{A}_{\V{\theta}^\star}(\V{x}^\star-\V{x})=\V{0} \}$ and $R$ be the basin of attraction of $(\V{\theta}^{\star},S)$. Let $S_1, S_2 \subset S$ be two arbitrary subsets of $S$ and $R_1, R_2 \subset R$ be their respective basins of attraction. Then
\[
 \frac{\mu(S_1)}{\mu(S_2)} = 0 \Rightarrow \frac{\mu(R_1)}{\mu(R_2)} = 0.
\]
\end{lemma}
In particular, if $\mu(S_1) = 0$ then $\mu(R_1) = 0$.  (Here $\mu(.)$ is defined on $S$ on the left and on the whole space on the right.)
\begin{proof}
For any $\V{\theta},\V{x} \in R$ define the iterates of gradient descent by $\left\{\V{\theta}^{(t)},\V{x}^{(t)}\right\}_{t = 0}^\infty$. Let $R_1(\V{\theta}),R_2(\V{\theta})$  be the restrictions of $R_1, R_2$ to $\V{x}$, that is $R_i(\V{\theta}) := \{\V{x}|(\V{\theta},\V{x})\in R_i\}$ and $S_1(\V{\theta}) \subset S_1,S_2(\V{\theta})\subset S_2$ be the stationary points corresponding to $R_1(\V{\theta}),R_2(\V{\theta})$.
Note that the iterates of gradient descent are given by 
\[
\V{x}^{(t+1)}=\V{x}^{(t)} + \eta \Mx{A}_{\V{\theta}^{(t)}}^T (\V{y}-\Mx{A}_{\V{\theta}^{(t)}}\V{x}^{(t)})= \underbrace{(\Mx{I}+\eta \Mx{A}_{\V{\theta}^{(t)}}^T \Mx{A}_{\V{\theta}^{(t)}})}_{:=\Mx{A}_t}\V{x}^{(t)})+\alpha \Mx{A}_{\V{\theta}^{(t)}}^T \V{y}.
\]
Since $\eta \|\Mx{A}_{\V{\theta}}\|^2 \leq 1 $, $\Mx{A}_t$ is full-rank. Therefore by Lemma \ref{lem:lin_det} for any $t$ we conclude
\[
\frac{\mu(R_1(\V{\theta}^{(t)}))}{\mu(R_2(\V{\theta}^{(t)}))} = \frac{\mu(R_1(\V{\theta}^{(t+1)}))}{\mu(R_2(\V{\theta}^{(t+1)}))}.
\]
Letting $t \rightarrow \infty$ we conclude
\[
\frac{\mu(R_1(\V{\theta}))}{\mu(R_2(\V{\theta}))} = \frac{\mu(S_1(\V{\theta}))}{\mu(S_2(\V{\theta}))}.
\]

\noindent By assumption 
\[
\frac{\mu(S_1)}{\mu(S_2)} =  \frac{\int_{R} \mu(S_1(\V{\theta})) d\mu\V{\theta}}{\int_{R} \mu(S_2(\V{\theta})) d\mu\V{\theta}}= 0.
\]
Therefore
\[
0 = \frac{\int_{R} \mu(S_1(\V{\theta})) d\mu\V{\theta}}{\int_{R} \mu(S_2(\V{\theta})) d\mu\V{\theta}} = \frac{\int_{R} \mu(S_2(\V{\theta})) \frac{\mu(R_1(\V{\theta}))}{\mu(R_2(\V{\theta}))} d\mu\V{\theta}}{\int_{R} \mu(S_2(\V{\theta})) d\mu\V{\theta}},
\]
from which we conclude 
\[
\frac{\mu(R_1(\V{\theta}))}{\mu(R_2(\V{\theta}))} = 0,
\]
almost everywhere. Therefore
\[
\frac{\mu(R_1)}{\mu(R_2)} =  \frac{\int_{R} \mu(R_1(\V{\theta})) d\mu\V{\theta}}{\int_{R} \mu(R_2(\V{\theta})) d\mu\V{\theta}} = \frac{\int_{R} \frac{\mu(R_1(\V{\theta}))}{\mu(R_2(\V{\theta}))}\mu(R_2(\V{\theta})) d\mu\V{\theta}}{\int_{R} \mu(R_2(\V{\theta})) d\mu\V{\theta}} =0.
\]
\end{proof}

\noindent Note that the condition $\eta \|\Mx{A}_{\V{\theta}}\|^2 \leq 1$ was sufficient but not necessary in the proof of Lemma \ref{lem:basin}. In particular, as long as $\Mx{A}_t$ is full-rank our assertion holds. 

%% file: app_optim.tex
\section{Preliminary Results} \label{app:prelim}
\renewcommand{\theequation}{\Alph{section}.\arabic{equation}}
\renewcommand{\thefigure}{\Alph{section}.\arabic{figure}}

As outlined in Section \ref{sec:notation}, we would like to study the following objective function:
\begin{align} \label{eq:L_def}
\mathcal{L}_k \left(\alpha, {\Mx{\lambda}, \Mx{Q}} \right) := \frac{1}{N}\sum_{n = 1}^N \Big( y_n -  \left( \Mx{Q} \Mx{\Lambda} \Mx{Q}^T + \alpha \Mx{I} \right) \bullet \Mx{X}_n \Big)^2,
\end{align}
and its regularized version
\begin{equation}
\mathcal{F}_k\left(\alpha, {\V{\lambda},\Mx{Q}}\right):= \mathcal{L}_k \left(\alpha, {\V{\lambda},\Mx{Q}} \right) + \gamma \left\|\Mx{QQ}^T-\Mx{I}\right\|_{\sf Fro},
\end{equation}
where $\Mx{X}_n \in \mathbb{R}^{d \times d}$, $y_n \in \mathbb{R}$, $\Mx{Q} \in \mathbb{R}^{d \times k}$ and $\Mx{\Lambda} = \diag(\V{\lambda}) \in \mathbb{R}^{k \times k}$ and $\alpha \in \R$. We are given $N$ observations $\{\Mx{X}_n, y_n\}_{n=1}^N$, In the case of quadratic neural networks, $\Mx{X}_n = \V{x}_n \V{x}_n^T$ is the outer product of the input vectors. Thus, $d$ is the input dimension of the observations, $k$ the number of hidden neurons, and $N$ the number of samples.
We wish to characterize the landscape of the optimization problem:
\begin{equation} \label{eq:mseloss}
\minimize_{\alpha,\V{\lambda}, \Mx{Q}} \mathcal{F}_k\left(\alpha, {\V{\lambda},\Mx{Q}}\right),
\end{equation}
as well as the its population counterpart
\begin{equation} \label{eq:mseloss_population}
\minimize_{\alpha,\V{\lambda}, \Mx{Q}} \mathbb{E} \;\mathcal{F}_k\left(\alpha, {\V{\lambda},\Mx{Q}}\right).
\end{equation}

We will also consider the following less restrictive objective:
\begin{equation}
\mathcal{L}_k(\alpha,\Mx{M},\Mx{Q}) := \frac{1}{N}\sum_{n=1}^N \Big( y_n -\left(\Mx{QMQ}^T + \alpha \Mx{I} \right) \bullet \Mx{X}_n \Big)^2,
\end{equation}
where $\Mx{M} \in \R^{k \times k}$ is not necessarily diagonal and consider the problem 
\begin{equation} \label{eq:mseloss_M}
\minimize_{\alpha,\Mx{M}, \Mx{Q}} \mathcal{F}_k \left(\alpha, {\Mx{M},\Mx{Q}}\right):= \mathcal{L}_k \left(\alpha, {\Mx{M},\Mx{Q}} \right) + \gamma \left\|\Mx{QQ}^T-\Mx{I}\right\|_{\sf Fro}.
\end{equation}

We denote the residual for observation $n$ as $r_n$:
\[
r_n := y_n - (\Mx{Q} \Mx{\Lambda} \Mx{Q}^T + \alpha \Mx{I}) \bullet \Mx{X}_n,
\]
or 
\[
r_n : =   y_n - (\Mx{QMQ}^T + \alpha \Mx{I}) \bullet \Mx{X}_n
\]
depending on the objective we are considering. 

Lastly, we can assume $\Mx{M},\Mx{X}_n$ are symmetric without loss of generality due to the relation: 
\[
\Mx{Q} \Mx{\Lambda} \Mx{Q}^T \bullet \Mx{X}_n = \Mx{Q} \Mx{\Lambda} \Mx{Q}^T \bullet \frac{\Mx{X}_n+\Mx{X}_n^T}{2},
\]
and assuming $\Mx{X}_n$'s are symmetric we have
\[
\Mx{Q} \Mx{M} \Mx{Q}^T \bullet \Mx{X}_n = \Mx{Q} \left(\Mx{\frac{\Mx{M}+\Mx{M}^T}{2}}\right) \Mx{Q}^T \bullet \Mx{X}_n.
\]

\subsection{Optimality Conditions}

We will next derive a number of optimality conditions for points on the loss landscape of $\mathcal{L}_k(\alpha,\V{\lambda}, \Mx{Q})$ (\ref{eq:mseloss}). We begin by stating the necessary and sufficient conditions for global optimality of a solution.
\begin{lemma}[Global Optimality Conditions] \label{lem:KKT}
For $k \geq d$, a point $\left(\alpha,\V{\lambda}, \Mx{Q} \right)$ is the global minimizer of $\mathcal{L}_k \left( \alpha,{\V{\lambda},\Mx{Q}} \right)$(\ref{eq:mseloss}) if and only if 
\begin{equation} \label{eq:KKT}
\sum_{n = 1}^N r_n \Mx{X}_n = \sum_{n = 1}^N \left( y_n -  \left( \Mx{Q} \Mx{\Lambda} \Mx{Q}^T +\alpha \Mx{I} \right)\bullet \Mx{X}_n  \right) \Mx{X}_n = {0}.
\end{equation}
\end{lemma}
\begin{proof}
Consider the optimization problem 
\begin{equation} \label{eq:cvx}
\minimize_{\Mx{A}} \mathcal{L}_{\sf cvx}(\Mx{A}):= \frac{1}{N}\sum_{n=1}^N \left(y_n - \Mx{A} \bullet \Mx{X}_n \right)^2,
\end{equation}
which is convex in $\Mx{A}$. The KKT optimality conditions at the optimal solution to this problem are given by:
\begin{equation} \label{eq:cvx_KKT}
\sum_{n = 1}^N \left( y_n -  \Mx{A} \bullet \Mx{X}_n  \right) \Mx{X}_n = {0}.
\end{equation}
Hence if $\left(\alpha, \V{\lambda}, \Mx{Q} \right)$ satisfies (\ref{eq:KKT}) it is globally optimal with $\Mx{A} = \Mx{Q} \Mx{\Lambda} \Mx{Q}^T +\alpha \Mx{I}$. Conversely, note that since $k \geq d$, if $\Mx{A}$ is an optimal solution to $\min_{\Mx{A}} \mathcal{L}_{\sf cvx}(\Mx{A})$, it has an eigenvalue decomposition of the form $\Mx{A} = \Mx{Q} \Mx{\Lambda} \Mx{Q}^T$. Therefore the global minimum of $\mathcal{L}_{\sf cvx}(\Mx{A})$ (\ref{eq:cvx}) is also achievable and the same as that of $\mathcal{L}_k \left(\alpha, {\V{\lambda},\Mx{Q}} \right)$ (\ref{eq:mseloss}).
\end{proof}

\begin{lemma}[First-Order Optimality Conditions] \label{lem:modified_kkt}
For any $k$, the first-order optimality conditions for a local minimum of $\mathcal{L} \left(\alpha, {\Mx{M},\Mx{Q}} \right)$ (\ref{eq:mseloss}) are given by:
\begin{align}
{\nabla}_{\Mx{Q}}  \mathcal{L} \left(\alpha, {\Mx{M},\Mx{Q}} \right)  &= \sum_{n=1}^N r_n \Mx{X}_n \Mx{Q} \Mx{M} = \V{0}_{d \times d}, \label{eq:first_order_optimality01} \\
{\nabla}_{\Mx{M}} \mathcal{L} \left(\alpha, {\Mx{M},\Mx{Q}} \right) &= \Mx{Q}^T \sum_{n=1}^N r_n \Mx{X}_n \Mx{Q} = \Mx{0}_{d \times d}, \label{eq:first_order_optimality02}\\
{\nabla}_{\Mx{\alpha}} \mathcal{L} \left(\alpha, {\Mx{M},\Mx{Q}} \right) & = \sum_{n=1}^N r_n \Mx{X}_n \bullet \Mx{I} = 0. \label{eq:optimality_alpha}
\end{align}
Moreover, if $\Mx{M}$ is full-rank, i.e. $\rank(\Mx{M}) =d$, (or equivalently the diagonal entries $\lambda_i = \Lambda_{ii}$ are non-zero), the optimality conditions for $\Mx{M}$ and $\Mx{Q}$ reduce to
\begin{align}\label{eq:first_order_optimality}
&\sum_{n=1}^N r_n \Mx{X}_n \Mx{Q} = \Mx{0}_{d \times d}.
\end{align}
\end{lemma}
\begin{proof}
Since $\mathcal{L} \left(\alpha, {\Mx{M},\Mx{Q}} \right)$ is differentiable with respect to its arguments, proof of the first statement is a direct consequence of setting the gradients equal to zero. For the second statement, if $\rank(\Mx{M}) = d$ we can multiply both sides of (\ref{eq:first_order_optimality01}) by $\Mx{M}^{-1}$ from which (\ref{eq:first_order_optimality}) follows.
\noindent Moreover (\ref{eq:first_order_optimality}) directly implies (\ref{eq:first_order_optimality02}).
\end{proof}

A simple intuition for Lemma \ref{lem:modified_kkt} is that when all the neurons are active ($\lambda_i \neq 0$), then their amplitudes can be absorbed into the corresponding quadratic layer weights ($\V{q}_i$'s). Therefore, optimization over $(\V{\lambda},\Mx{Q})$ is \textit{locally} equivalent to optimization over $\Mx{Q}$ only.

By comparing  Lemma \ref{lem:KKT} and (\ref{eq:first_order_optimality01}-\ref{eq:first_order_optimality02}) we immediately observe the following Lemma 
\begin{lemma} \label{lem:lr_stationary} For any $k$ and arbitary fixed $\alpha$ all critical points of $\mathcal{L}(\alpha,\Mx{M},\Mx{Q})$ which are not a global minimum should satisfy
\begin{equation}
    \rank(\Mx{Q}) < d.
\end{equation}
\end{lemma}

\begin{lemma}[Second-Order Optimality Conditions] \label{lem:modified_kkt_second_order}
For a local minimum or saddle point with no negative curvature direction $(\alpha,\Mx{M}, \Mx{Q})$ of $\mathcal{L} \left( {\alpha, \Mx{M},\Mx{Q}} \right)$ (\ref{eq:mseloss}) the second-order optimality condition is given by
\begin{align} \label{eq:second_order_optimality}
\sum_{n=1}^N (\Mx{X}_n \Mx{Q} \Mx{M} \bullet \Mx{U})^2 -\sum_{n=1}^N r_n \Mx{X}_n \bullet \Mx{U} \Mx{M} \Mx{U}^T \geq 0,
\end{align}
for all $\Mx{U} \in \R^{d \times d}$.
\end{lemma}


\noindent Note that the global solution $\Mx{A}^*$ to the optimization problem $\min_{\Mx{A}} \mathcal{L}_{\sf cvx}(\Mx{A})$ can be found by a linear regression (least squares) by vectoring the upper triangular portion of the symmetric matrix. Therefore, an easy but not computationally efficient solution to minimizing $\mathcal{L}_k \left(\alpha, {\V{\lambda},\Mx{Q}} \right)$ can be obtained by finding the eigendecomposition of the solution: $\Mx{A}^* = \Mx{Q}^* \Mx{\Lambda}^* \Mx{Q}^{*T}$. However, in most implementations eigenvalue decompositions have a computational cost of $\mathcal{O}(d^3)$, which is restrictive for high-dimensional inputs in contrast to local search methods.

This alternative algorithm does suggest that we will not need more than $k = d$ neurons to achieve the global minima; the eigenvalue decomposition of a full-rank $\Mx{A}^*$ gives $\Mx{Q}^* \in \R^{d \times d}$ which only requires $d$ hidden neurons to represent. We will formalize this statement in the next section.

\subsection{Number of Hidden Units} 
Our first goal is to understand how varying the number of hidden units $k$ with respect to the input size $d$ affects the properties of the loss landscape; in particular, how many neurons are required for the landscape to contain the smallest loss achievable by a two-layer network?  

\begin{lemma}[Sufficiency of $k =d$ neurons to reach global minima] \label{lem:k_equals_d}
Let $(\alpha^{(k)},\Mx{\Lambda}^{(k)},\Mx{Q}^{(k)})$ be the global minimum to $\mathcal{L}_k\left(\alpha, \V{\lambda}, \Mx{Q} \right)$, then for $k\geq d$ the value of the objective function at the global minimum $\mathcal{L}_k\left(\alpha^{(k)}, \V{\lambda}^{(k)}, \Mx{Q}^{(k)} \right)$ does not depend on $k$. Moreover if $\gamma > 0 $ then the global minimum satisfies $\Mx{QQ}^T = \Mx{I}$. 
\end{lemma}
\begin{proof}
Let $\Mx{M} \in \R^{d \times d}$ be an arbitrary matrix and define $z_n = y_n -\Mx{M} \bullet \Mx{X}_n$, i.e. the residual for observation $n$ using matrix \Mx{M}. Define the optimization problem:
\begin{align} \label{eq:z_n}
    \minimize_{\Mx{B}} \frac{1}{N}\sum_{n=1}^N (z_n - \Mx{B} \bullet \Mx{X}_n)^2.
\end{align}
Note that $\Mx{B}$ is a global minimum to (\ref{eq:z_n}) if and only if $\Mx{A} = \Mx{B} + \Mx{M}$ is a global minimum to 
\begin{align}
    \minimize_{\Mx{A}} \frac{1}{N}\sum_{n=1}^N (y_n - \Mx{A} \bullet \Mx{X}_n)^2.
\end{align}
Now let $(\alpha^{(d)},\V{\lambda}^{(d)}, \Mx{Q}^{(d)})$ and $(\alpha^{(k)},\V{\lambda}^{(k)}, \Mx{Q}^{(k)})$ be global minima to (\ref{eq:mseloss}) for $d$ and $k>d$ neurons respectively and define 
\[
\Mx{M} = \sum_{i = d+1}^k \lambda_i^{(k)} \V{q}_i^{(k)}\V{q}_i^{(k)^T}+ \left(\alpha^{(k)}-\alpha^{(d)}\right) \Mx{I}.
\]
Therefore, by fixing $\alpha = \alpha^{(d)}$ (or any other value for that matter), the global minima $(\lambda_i^{(k)}, \V{q}_i^{(k)})$, $i = 1, \cdots, d$ of (\ref{eq:mseloss}) for $k = d$, can be obtained by an eigenvalue decomposition of the following form:
\[
\Mx{B} = \underbrace{\sum_{i = 1}^d \lambda_i^{(d)} \V{q}_i^{(d)} \V{q}_i^{(d)^T}}_{\Mx{A}}-\left(\underbrace{\sum_{i = d+1}^k \lambda_i^{(k)} \V{q}_i^{(k)} \V{q}_i^{(k)^T} + \left(\alpha^{(k)}-\alpha^{(d)}\right) \Mx{I} }_{\Mx{M}}\right).
\]
\end{proof}

A simple intuition for Lemma \ref{lem:k_equals_d} is that the rank of a matrix $\Mx{A} \in \R^{d \times d}$ cannot exceed $d$, therefore having more than $d$ neurons is not going to improve the global minimum. Unfortunately, this also means that that we cannot hope for expressivity of a single layer quadratic network to improve by simply increasing the number of neurons. We will see later that this problem can be improved using a deep quadratic network instead.
In the rest of the appendix we will assume $k = d$ and drop the subscript in $\mathcal{L}_k$ whenever the number of neurons is clear from the context.

We will find it useful throughout to work with the objective $\mathcal{L}(\alpha, \Mx{M},\Mx{Q})$ which does not require the diagonal matrix $\Mx{\Lambda}$.  The following lemma however, asserts that we can map points between the landscapes of $\mathcal{L}(\alpha, \Mx{\Lambda},\Mx{Q})$ and $\mathcal{L}(\alpha, \Mx{M},\Mx{Q})$, and thus the two landscapes are equivalent.
\begin{lemma} \label{lem:equivalence} The optimization problem (\ref{eq:mseloss}) is equivalent to 
\begin{equation}
    \minimize_{\alpha,\Mx{M}, \Mx{Q}} \mathcal{L}_k(\alpha,\Mx{M},\Mx{Q}),
\end{equation}
that is
\[
\mathcal{L}_k(\alpha,\Mx{M},\Mx{Q}) = \mathcal{L}_k(\alpha,\Mx{\Lambda},\Mx{Q}).
\]
Moreover, for two equivalent points of these two optimization problems we have 
\[
\left|\supp(\Mx{\lambda})\right| = \rank(\Mx{M}).
\]

\end{lemma}
\begin{proof}
Let $\Mx{M}$ have an eigenvalue decomposition of the form $\Mx{M} = \Mx{U} \Mx{\Lambda} \Mx{U}^T$. We have the equivalence
\begin{align*}
\mathcal{F}_k\left(\alpha, \Mx{M}, \Mx{Q} \right) = \mathcal{L}_k\left(\alpha, \Mx{M}, \Mx{Q} \right)+ \gamma \left\|\Mx{QQ}^T-\Mx{I}\right\|^2_{\sf Fro} &= \mathcal{L}_k\left(\alpha, \Mx{U\Lambda U}^T, \Mx{Q} \right) + \gamma \left\|\Mx{Q}\Mx{Q}^T-\Mx{I}\right\|^2_{\sf Fro}\\
&=  \mathcal{L}_k\left(\alpha, \Mx{\Lambda}, \Mx{QU} \right)+ \gamma \left\|\Mx{QUU}^T\Mx{Q}^T-\Mx{I}\right\|^2_{\sf Fro} \\
&= \mathcal{F}_k\left(\alpha, \Mx{\Lambda}, \Mx{QU} \right).
\end{align*}
Since, $(\Mx{\Lambda}, \Mx{U})$ are continuous functions in $\Mx{M}$ (Lemma \ref{cor:continuous_eig}) we conclude $(\Mx{M},\Mx{Q})$ is a stationary point of (\ref{eq:mseloss}) if and only if $(\Mx{\Lambda},\Mx{QU})$ is a stationary point of (\ref{eq:mseloss_M}). Since both (\ref{eq:mseloss}) and (\ref{eq:mseloss_M}) are convex in $(\alpha,\Mx{\Lambda})$ and $(\alpha,\Mx{M})$ given $\Mx{Q}$ respectively it only remains to check for second-order optimality for $\Mx{Q}$ in both problems. Let $\Mx{V}\in \R^{d \times k}$ be an arbitrary matrix. Then we have
\begin{align*}
    \Mx{\nabla}^2_{\Mx{Q}}\Big(\mathcal{F}_k\left(\alpha, \Mx{M}, \Mx{Q} \right) \Big) \bullet \Big( \Mx{V \otimes V} \Big) &= \sum_{n=1}^N (\Mx{X}_n \Mx{Q} \Mx{M} \bullet \Mx{V})^2 -\sum_{n=1}^N r_n \Mx{X}_n \bullet \Mx{V} \Mx{M} \Mx{V}^T \\
    &= \sum_{n=1}^N (\Mx{X}_n \Mx{Q} \Mx{U \Lambda} \bullet \Mx{VU})^2 -\sum_{n=1}^N r_n \Mx{X}_n \bullet \Mx{VU} \Mx{\Lambda U}^T \Mx{V}^T \\
    & =  \Mx{\nabla}^2_{\Mx{Q}}\Big(\mathcal{F}_k\left(\alpha, \Mx{\Lambda}, \Mx{QU} \right) \Big) \bullet \Big( \Mx{VU \otimes VU} \Big).
\end{align*}
Therefore there is a one to one mapping between negative/positive curavature directions of  $\mathcal{F}_k\left(\alpha, \Mx{M}, \Mx{Q} \right)$ and $\mathcal{F}_k\left(\alpha, \Mx{\Lambda}, \Mx{Q} \right)$, which proves the assertion.

\end{proof}

\noindent \textbf{Remark:} There is an equivalence of the landscapes of $\mathcal{L}(\alpha,\V{\lambda},\Mx{Q})$ and $\mathcal{L}(\alpha,\Mx{M},\Mx{Q})$ in the sense that every point $(\alpha,\V{\lambda},\Mx{Q})$ is mapped to a unit sphere (or equivalently the orthogonal group) corresponding to different orthonormal matrices $\Mx{U}$. 


Note that Lemma \ref{lem:equivalence} is only true when the network output is scalar. In particular, in the multivariate output case we will have different shared quadratic dictionaries $\Mx{Q}$ and different $\Mx{\Lambda}_i$'s for each output. Since matrices $\Mx{M}_i$ are not in general simultaneously diagonalizable we cannot replace $\Mx{\Lambda}_i$'s by their less constrained counterparts $\Mx{M}_i$. The problem of multivariate outputs will be dealt with using a different trick in more detail in Section \ref{sec:multivariate}.

\begin{lemma} \label{lem:Q_diag}
For a local minimum $(\Mx{M}^\star,\Mx{Q}^\star)$, without loss of generality we can assume $\Mx{Q}$ is diagonal.
\end{lemma}
\begin{proof}
Let $\Mx{Q}^\star=\Mx{U\Sigma V}^T$ be the SVD of $\Mx{Q}$. As alluded to in the proof of Lemma \ref{lem:equivalence} the linear transformation $\Mx{Q} \rightarrow \Mx{QV}, \Mx{M} \rightarrow \Mx{V}^T\Mx{MV}$ does not change the landscape of $\mathcal{F}$. Moreover, $(\Mx{M}^\star,\Mx{Q}^\star)$ is a stationary point of (\ref{eq:mseloss}) for the data $\left\{y_n,\Mx{X}_n\right\}_{n=1}^N$ if and only if $(\Mx{M}^\star,\Mx{U}^T\Mx{Q}^\star)$ is a stationary point of (\ref{eq:mseloss}) for the data $\left\{y_n,\Mx{UX}_n\Mx{U}^T\right\}_{n=1}^N$.
\end{proof}

\begin{definition}\label{def:SQ}
Let $(\alpha,\Mx{M},\Mx{Q})$ be stationary point, i.e. a point satisfying the first order optimality conditions (\ref{eq:first_order_optimality01}-\ref{eq:optimality_alpha}) such that $\rank(\Mx{Q}) = r$ and define the set
\[
\mathcal{M}_{\Mx{Q}} : = \left\{ \widetilde{\Mx{M}} \in \R^{d \times d} \Bigg|
\Mx{Q}\left(\Mx{\widetilde{M}}-\Mx{M}\right) = \Mx{0}, \widetilde{\Mx{M}} = \widetilde{\Mx{M}}^T\right\}.
\]
Moreover define the subsets
\[
\mathcal{M}^-_{\Mx{Q}} := \left\{ \widetilde{\Mx{M}} \in \mathcal{M}_{\Mx{Q}} \Big| (\alpha,\Mx{Q},\widetilde{\Mx{M}}) \text{ is a saddle point with a negative curvature direction} \right\},
\]
and
\[
\mathcal{M}^+_{\Mx{Q}} := \left\{ \widetilde{\Mx{M}} \in \mathcal{M}_{\Mx{Q}} \Big| (\alpha,\Mx{Q},\widetilde{\Mx{M}}) \text{ satisfies the second-order optimality conditions (\ref{eq:second_order_optimality}}) \right\}.
\]
\end{definition}

\noindent  Note that by definition $\mathcal{M}_{\Mx{Q}}$ is connected and isomorphic to $\R^{\binom{d+1}{2}-\binom{r+1}{2}}$.

\begin{lemma} Let $(\alpha,\Mx{M},\Mx{Q})$ be a stationary point and $\widetilde{\Mx{M}} \in \mathcal{M}_{\Mx{Q}}$ be an arbitrary point. Then $\mathcal{L}(\alpha, \widetilde{\Mx{M}},\Mx{Q}) = \mathcal{L}(\alpha, \Mx{M},\Mx{Q})$ and $(\alpha, \widetilde{\Mx{M}},\Mx{Q})$ satisfies the first order optimality conditions.
\end{lemma}
\begin{proof}
Note that $\Mx{Q}\Mx{M} = \Mx{Q}\widetilde{\Mx{M}}$ by definition. Therefore  $\mathcal{L}(\alpha, \widetilde{\Mx{M}},\Mx{Q}) = \mathcal{L}(\alpha, \Mx {M},\Mx{Q})$ and all first-order optimality conditions follow.
\end{proof}
\begin{corollary} We have
\[
\mathcal{M}_{\Mx{Q}} = \mathcal{M}^-_{\Mx{Q}} \cup \mathcal{M}^+_{\Mx{Q}}.
\]
\end{corollary}
\begin{lemma} \label{lem:measure_0} Let $\mu(.)$ be the Lebesgue measure defined on $\mathcal{M}_{\Mx{Q}}$. Then one of the following statements is true
\begin{enumerate}
    \item 
    \[
    \mu \left(\Mx{\mathcal{M}^+_{\Mx{Q}}} \right) =0,
    \]
    \item $\sum_{n=1}^N r_n \Mx{X}_n $ is semidefinite, i.e.
    \[
    \sum_{n=1}^N r_n \Mx{X}_n \succeq \Mx{0},
    \] or
    \[
    \sum_{n=1}^N r_n \Mx{X}_n \preceq \Mx{0}.
    \]
\end{enumerate}
\end{lemma}

\begin{proof}
If $(\alpha,\Mx{M},\Mx{Q})$ is a global minimum then by the global optimality conditions (Lemma \ref{lem:KKT}), $\sum_{n=1}^N r_n \Mx{X}_n = \Mx{0}$ and the claim follows. Therefore, without loss of generality we can assume $(\alpha,\Mx{M},\Mx{Q})$ is not a global minimum.
By Lemma \ref{lem:lr_stationary} we have $r = \rank(\Mx{Q}) < d$. Therefore $\mathcal{M}_{\Mx{Q}}$ is not a singleton and $\rank(\Mx{QM}) < d$. Choose a nonzero $\V{v} \in \mathcal{N}(\Mx{Q M})$, i.e. such that $\Mx{Q M}\V{v} = \Mx{0}$.
Taking $\Mx{U} = \V{uv}^T$ with $\V{u}$ arbitrary in the second-order optimality conditions (\ref{eq:second_order_optimality}) for $\Mx{M} \in \mathcal{M}^+_{\Mx{Q}}$ we obtain
\begin{equation} \label{eq:semidefinite}
    -\V{u}^T \sum_{n = 1}^N r_n \Mx{X}_n  \V{u v}^T \Mx{M} \V{v} \leq \Mx{0},   
\end{equation}
for all $\V{u} \in \R^d$. If $\mu(\mathcal{M}^+_{\Mx{Q}})>0$, then there exists an $\widetilde{\Mx{M}}\in \mathcal{M}^+_{\Mx{Q}}$ such that $ \V{v}^T \widetilde{\Mx{M}} \V{v} \neq \Mx{0}$, hence by (\ref{eq:semidefinite}), we conclude $\sum_{n=1}^N r_n \Mx{X}_n$ should be semi-definite.
\end{proof}

The following Lemma shows that the choice of step-size in gradient descent is somewhat arbitrary in (\ref{eq:mseloss}).
\begin{lemma} \label{lem:step_size} Let $\gamma = 0$ and $(\V{\lambda}_0,\Mx{Q_0})$ be an initialization of gradient descent for (\ref{eq:mseloss}), $\beta >0$ an arbitrary constant and $(\V{\lambda}^\star,\Mx{Q}^\star)$ the corresponding stationary point of gradient descent. Then
\begin{enumerate}
    \item $(\beta^2 \V{\lambda}^\star,\frac{\Mx{Q}^\star}{\beta})$ is a stationary point of the same type of $(\V{\lambda}^\star,\Mx{Q}^\star)$, i.e. they are both global/local minima/maxima or saddle points with at least one/no negative curvature direction.
    \item If $(\eta_{\Mx{Q}},\eta_{\V{\lambda}})$ are the step-sizes corresponding to $(\V{\lambda}_0,\Mx{Q_0})$, then using the initialization  $(\beta^2\V{\lambda}_0,\frac{\Mx{Q_0}}{\beta})$ and step-sizes $(\eta_{\Mx{Q}}/\beta^2,\beta^4\eta_{\V{\lambda}})$ gradient descent reaches $(\beta^2 \V{\lambda}^\star,\frac{\Mx{Q}^\star}{\beta})$ (which is essentially the same as $(\V{\lambda}^\star,\Mx{Q}^\star)$).
\end{enumerate}
\end{lemma}
\begin{proof}
Part 1 is easily verified and we omit the proof. Let $(\widetilde{\V{\lambda}},\widetilde{\Mx{Q}}):=(\beta^2\V{\lambda}_0,\frac{\Mx{Q_0}}{\beta})$. It is similarly straight-forward to show that $\mathcal{F}(\alpha,\V{\lambda},\Mx{Q}) = \mathcal{F}(\alpha,\widetilde{\V{\lambda}},\widetilde{\Mx{Q}})$, and
\[
\begin{array}{l}
     \Mx{\nabla}_{\Mx{Q}} \mathcal{F}(\alpha,\V{\lambda},\Mx{Q}) = {\Mx{\nabla}_{\Mx{Q}} \mathcal{F}(\alpha,\widetilde{\V{\lambda}},\widetilde{\Mx{Q}})}/{\beta},  \\
     \Mx{\nabla}_{\V{\lambda}} \mathcal{F}(\alpha,\V{\lambda},\Mx{Q}) = \beta^2{\Mx{\nabla}_{\Mx{\lambda}} \mathcal{F}(\alpha,\widetilde{\V{\lambda}},\widetilde{\Mx{Q}})}. 
\end{array}
\]
Therefore, the gradient decent iterates are given by
\[
\begin{array}{l}
\widetilde{\Mx{Q}}_{t+1} = \widetilde{\Mx{Q}}_{t} - \eta_{\Mx{Q}}/\beta^2 \Mx{\nabla}_{\Mx{Q}} \mathcal{F}(\alpha,\widetilde{\V{\lambda}}_t,\widetilde{\Mx{Q}}_t)  =\widetilde{\Mx{Q}}_t/\beta, \\
\widetilde{\V{\lambda}}_{t+1} = \widetilde{\V{\lambda}}_{t} - \beta^4 \eta_{\V{\lambda}}\Mx{\nabla}_{\Mx{\lambda}} \mathcal{F}(\alpha,\widetilde{\V{\lambda}}_t,\widetilde{\Mx{Q}}_t) =\beta^2 {\V{\lambda}}_t,
\end{array}
\]
which proves the assertion.
\end{proof}

\begin{corollary} \label{cor:step_size}
For any choice of step-size step-size for gradient descent in (\ref{eq:mseloss}), there exists constant $g_{\eta}, h_{\eta}$ such that if $\|\V{\lambda}_0\| < g_{\eta}$ and $\|\Mx{Q}_0\| \geq h_{\eta}$, using the initialization $(\V{\lambda}_0 , \Mx{Q}_0)$ the spectral condition of Lemma \ref{lem:basin} is met.
\end{corollary}
\begin{proof}
Note that (\ref{eq:mseloss}) is equivalent to \ref{eq:theta_x} with $\V{x} = \V{\lambda}$, $\V{y}= [y_1,\cdots,y_N]^T$, $\V{\theta}= \Mx{Q}$ and
\[
\Mx{A}_{\V{\theta}} = \Big(\Mx{Q} \left[\V{x}_1,\cdots,\V{x}_N \right] \Big)^{\odot 2^T}.
\]
The assertion follows from the fact that $\beta$ can be chosen to be arbitrarily small in Lemma \ref{lem:step_size} and that
$\eta \| \Mx{A}_{\theta}\|^2$ scales linearly with $\beta^2$.
\end{proof}

\begin{lemma}\label{lem:bounded_eig} There exists a constant $\lambda^{\star}_{\max}$ such that for any stationary point of (\ref{eq:mseloss}) we have
\[
\lambda_{\max}\left(\sum_{n=1}^N{r_n \Mx{X}_n}\right) \leq \lambda_{\max}^\star.\]
\end{lemma}
\begin{proof}
Note that by convexity of (\ref{eq:mseloss}) in $(\alpha,\Mx{M})$ for any stationary point of (\ref{eq:mseloss}) we have 
\[ 
\sum_{n=1}^N r_n^2 = \mathcal{F}(\alpha^\star, \Mx{M}^\star, \Mx{Q}^\star) \leq \mathcal{F}(0, \Mx{0}, \Mx{Q}^\star) = \sum_{n=1}^N y_n^2.
\]
Defining $B : = B(\Mx{0},\|\V{y}\|_2) = \{\V{r} \in \R^N| \|\V{r}\|_2 \leq \|\V{y}\|_2\}$ we conclude
\[
\lambda_{\max}\left(\sum_{n=1}^N{r_n \Mx{X}_n}\right) \leq \sup_{B} \lambda_{\max}\left(\sum_{n=1}^N{z_n \Mx{X}_n}\right).\]
\end{proof}




%% file: app_proof_alpha.tex
\section{Proof of the Main Theorems}
\renewcommand{\theequation}{\Alph{section}.\arabic{equation}}
\renewcommand{\thefigure}{\Alph{section}.\arabic{figure}}

In this section we prove the main theoretical result of this paper for a two-layer neural network with quadratic activation functions. We restate this main theorem here for completeness

\begin{theorem} [Main Theorem on Single-Layer Quadratic-Linear Networks] \label{thm:main-alpha} For $k \geq d$ the following statements hold for gradient descent with step-size $\eta$: 
\begin{enumerate}
    \item For $\gamma = 0$, there exist constants $g_\eta, h_\eta$ such that using the initialization $(\|\V{\lambda}_0\| \leq g_{\eta},\|\Mx{Q}_0\| \geq h_{\eta}$) which are drawn from an arbitrary absolutely continuous distribution, with probability 1 gradient descent goes to a stationary point of $\mathcal{L}(\alpha,\V{\lambda},\Mx{Q})$ which is either a global minimum or has a negative curvature direction. In particular stochastic gradient descent will escape such points and reach a global minimum of $\mathcal{L}(\alpha,\V{\lambda},\Mx{Q})$.
    \item For $\alpha = 0$, if $\gamma > \frac{1}{N}\sum_{n=1}^N y_n^2$, with the initialization $(\V{\lambda}_0 = \V{0},\Mx{Q}_0\Mx{Q}_0^T = \Mx{I})$ gradient descent will converge to a global minimum of (\ref{eq:mseloss}). Moreover, $\Mx{Q}\Mx{Q}^T = \Mx{I}$.
    \item For $\alpha = 0$ and any initialization $(\V{\lambda}_0,\Mx{Q}_0)$ there exists a constant $\gamma_{0}$ such that for $\gamma \geq \gamma_0$ any stationary point of gradient descent will either have a negative curvature direction or is a global minimum of (\ref{eq:mseloss}). 
\end{enumerate}
\end{theorem}
\begin{proof} 
In order to prove the first statement, let $(\alpha,\Mx{M},\Mx{Q})$ be a stationary point of $\mathcal{L}(\alpha,\Mx{M},\Mx{Q})$. By Lemma \ref{lem:measure_0} one of two cases happen.

\textbf{Case 1}: $\mu(\mathcal{M}^+_{\Mx{Q}}) = 0$. Define $R_{\Mx{Q}}^+$ to be the basin of attraction of $(\Mx{Q},\mathcal{M}^+_{\Mx{Q}})$. In this case by Corollary \ref{lem:step_size} the spectral condition of Lemma \ref{lem:basin} is met. Therefore, by Lemma \ref{lem:basin} $\mu(R_{\Mx{Q}}^+)=0$,3 and with probability 1 over the initialization, the Hessian will have a negative curvature direction.

\textbf{Case 2}: $\sum_{n=1}^N r_n \Mx{X}_n$ is semidefinite. Without loss of generality assume $\sum_{n=1}^N r_n \Mx{X}_n \succeq \Mx{0}$. Then by the optimality conditions for $\alpha$ (\ref{eq:optimality_alpha}) we have
\begin{align*}
    0 = \sum_{n=1}^N r_n \Mx{X}_n \bullet \Mx{I} = \trace(\sum_{n=1}^N r_n \Mx{X}_n) \geq 0,
\end{align*}
with equality if and only if $\sum_{n=1}^N r_n \Mx{X}_n = \Mx{0}$, which by Lemma \ref{lem:KKT} is the global optimality condition. Figure \ref{fig:MQ} shows a comparison of Cases 1 and 2.

\begin{figure}[H]
\begin{center}
\noindent
\includegraphics[trim = 5mm 0 0 0, clip=false, width=5in]{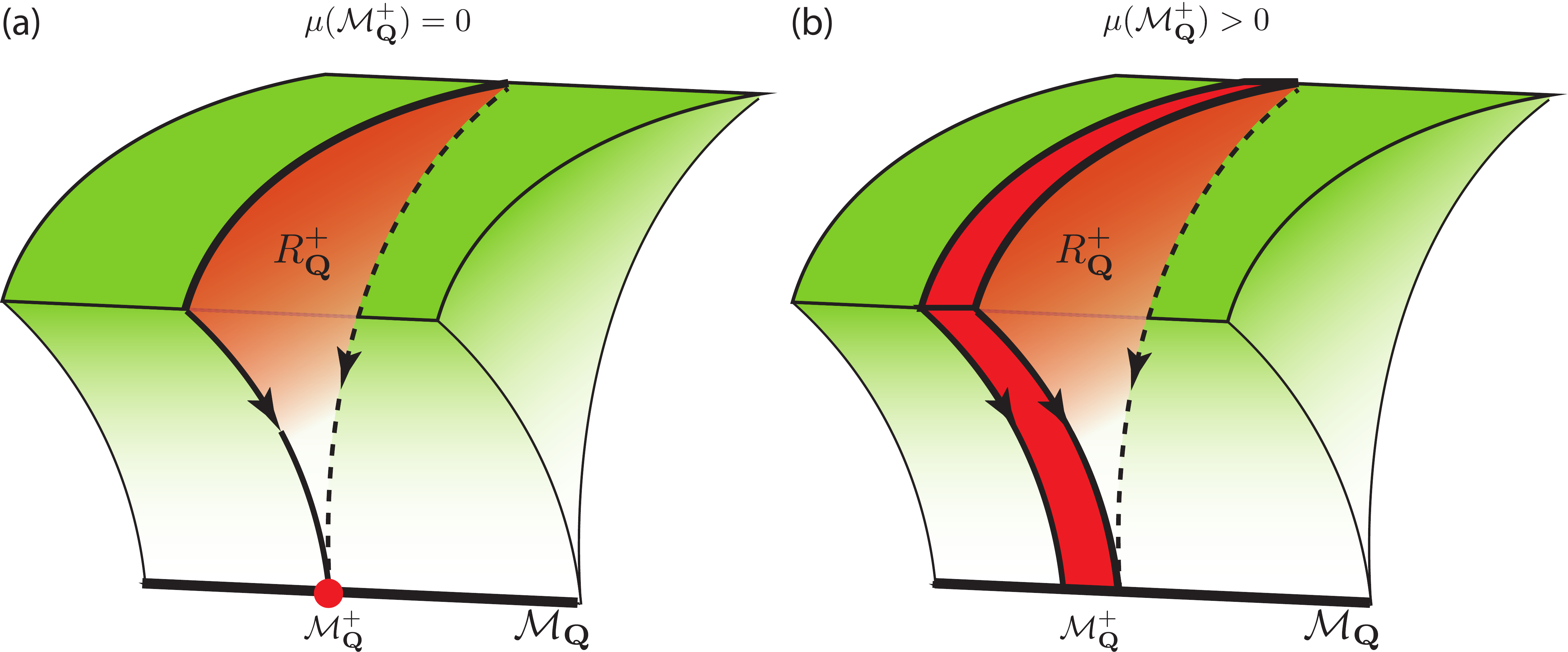}
\caption{Geometric intuition into the landscape of ($\lmq$). A stationary point $\Mx{Q}$ is associated with a subspace $\mathcal{M}_{\Mx{Q}}$. a) When $\mu(\mathcal{M}_{\Mx{Q}}^+)=0$ with probability 1 all saddle points have a negative curvature direction b) $\mu(\mathcal{M}_{\Mx{Q}}^+) > 0$, $\sum_{n=1}^N r_n \Mx{X}_n$ is semi-definite. Arrows represent the gradient flow.}
\label{fig:MQ}
\end{center}
\end{figure}

For the second statement, let $(\V{\lambda},\Mx{Q})$ be a stationary point of (\ref{eq:mseloss}). If $\mathcal{L}(\V{\lambda}, \Mx{Q}) = 0$ the proof is complete. Therefore without loss of generality assume $\mathcal{L}(\V{\lambda}, \Mx{Q}) > 0$. If $\rank(\Mx{Q}) < d $, then by Lemma \ref{lem:qq_lb} we have
\[
\left\|\Mx{QQ}^T - \Mx{I} \right\|_{\sf Fro}^2 \geq 1.
\]
Therefore
\[
\mathcal{L}(\V{\lambda},\Mx{Q}) + \gamma \left\|\Mx{QQ}^T-\Mx{I}\right\|^2_{\sf Fro} > \gamma = \frac{1}{N}\sum_{n}^N y_n^2 = \mathcal{L}(\V{\lambda}_0,\Mx{Q}_0),
\]
which is in contradiction to descent property of gradient descent. Therefore, we conclude $\rank(\Mx{Q}) = d$. Using the optimality condition for $\Mx{M}$ (\ref{eq:first_order_optimality02}) we conclude $\sum_{n=1}^N r_n \Mx{X}_n = \Mx{0}$ which is the global optimality condition for a minimizer of $\mathcal{L}(\V{\lambda},\Mx{Q})$. Moreover $\nabla_{\Mx{Q}} \mathcal{L}(\V{\lambda},\Mx{Q}) = \Mx{0}$. Therefore, 
\[\nabla_{\Mx{Q}} \left\|\Mx{QQ}^T - \Mx{I} \right\|_{\sf Fro}^2= 2(\Mx{QQ}^T-\Mx{I})\Mx{Q}=\Mx{0},
\]
which together with $\rank(\Mx{Q}) = d$, results in $\Mx{QQ}^T=\Mx{I}$.

In order to prove the final assertion, first note that when $\rank(\Mx{Q}) = d$ the proof is similar to the previous case. Therefore without loss of generality we assume  $ r = \rank(\Mx{Q}) <d$. Moreover, by Lemma \ref{lem:Q_diag} without loss of generality we can assume $\Mx{Q}$ is diagonal, i.e. $\Mx{Q}= \diag(\sigma_1,\sigma_2,\cdots,\sigma_r,0,\cdots,0)$. The second-order optimality condition for $\Mx{Q}$ is given by
\begin{equation} \label{eq:second_order_gamma}
\sum_{n=1}^N \left( \Mx{X}_n \bullet \Mx{Q}\Mx{M U}^T \right)^2 - \sum_{n = 1}^N r_n \Mx{X}_n \bullet \Mx{UMU}^T +\gamma\left( \left\|\Mx{QU}^T \right\|_{
    \sf Fro}^2+ 2\left(\Mx{Q}\bullet \Mx{U}\right)^2 - \|\Mx{U}\|_{\sf Fro}^2 \right) \geq 0,    
\end{equation}
for all $\Mx{U}\in \R^{d \times d}$. Let $\V{u} \in \R^d$ be a unit norm vector in the null space of $\Mx{Q}$, i.e. $u_i =0$ for $i =1,\cdots,r$ and $\V{v} \in \R^d$ be a unit norm vector and let $\Mx{U} = \V{uv}^T$. Then the left first, third and the fourth term in the left hand side of (\ref{eq:second_order_gamma}) are equal to 0 and the left hand side of (\ref{eq:second_order_gamma}) is given by
\[
-\left(\V{v}^T\sum_{n=1}^N {r_n}\Mx{X}_n \V{v} \right) \left( \V{u}^T \Mx{M} \V{u}\right) - \gamma (d-r) \geq 0.
\]
Moreover, by Lemma \ref{lem:bounded_eig} we have
\[
\lambda^{\star}_{\max} \lambda_{\max}(\Mx{M})- \gamma (d-r) \geq -\left(\V{v}^T\sum_{n=1}^N {r_n}\Mx{X}_n \V{v} \right) \left( \V{u}^T \Mx{M} \V{u}\right) - \gamma (d-r),
\]
which is negative if $\lambda_{\max}(\Mx{M}) < \gamma (d-r)/\lambda^{\star}_{\max}$.

Finally, note that the set $\{\Mx{M} | \lambda_{\max}(\Mx{M}) < \gamma (d-r)/ \lambda^{\star}_{\max} \}$ includes all of $\mathcal{M}_{\Mx{Q}}$ in its basin of attraction as $\gamma \rightarrow \infty$, which proves the assertion.
\end{proof}

The following Corollary is a restatement of Theorem \ref{thm:main-alpha} for population risk. The proof of this theorem is exactly the same except for replacing empirical terms with their population counterparts, and thus is omitted to avoid repetition.
\begin{corollary} [Main Theorem Extended to Population Risk] \label{thm:main-alpha-population} For $k \geq d$ the following statements hold for gradient descent with step-size $\eta$: 
\begin{enumerate}
    \item For $\gamma = 0$, there exist constants $g_\eta, h_\eta$ such that using the initialization $(\|\V{\lambda}_0\| \leq g_{\eta},\|\Mx{Q}_0\| \geq h_{\eta}$) which are drawn from an arbitrary absolutely continuous distribution, with probability 1 gradient descent goes to a stationary point of $\Bar{\mathcal{L}}(\alpha,\V{\lambda},\Mx{Q})$ which is either a global minimum or has a negative curvature direction.
    \item For $\alpha = 0$, if $\gamma > \mathbb{E}y^2$, with the initialization $(\V{\lambda}_0 = \V{0},\Mx{Q}_0\Mx{Q}_0^T = \Mx{I})$ gradient descent will converge to a global minimum of (\ref{eq:mseloss_population}). Moreover, $\Mx{Q}\Mx{Q}^T = \Mx{I}$.
    \item For $\alpha = 0$ and any initialization $(\V{\lambda}_0,\Mx{Q}_0)$ there exists a constant $\gamma_{0}$ such that for $\gamma \geq \gamma_0$ any stationary point of gradient descent will either have a negative curvature direction or is a global minimum of (\ref{eq:mseloss_population}). 
\end{enumerate}
\end{corollary}

\textbf{Remark 1:} An interesting observation made throughout the proof of Theorem \ref{thm:main-alpha} is that not only the stationary points satisfy $\rank(\Mx{Q}) = d$ but also the optimization path never goes through a low-rank $\Mx{Q}$.  Another interesting observation in the case $\alpha = \gamma = 0$ is that almost all bad stationary points  in the landscape of $\lmq$ satisfy $\rank(\Mx{Q}) = d-1$ which results in
\[
\rank(\sum_{n}r_n \Mx{X}_n) = 1.
\] 
This phenomenon can be intuitively explained as follows. First, note that it is highly unlikely that the rank of $\Mx{Q}$ decreases below $d-1$ via the gradient steps. Therefore such a stationary point is not escaped we conclude $\rank(\Mx{Q}) = d-1$. By what we showed throughout the proof of Theorem \ref{thm:main-alpha} most such points will be escaped. However the gradient updates for $\Mx{Q}$ is given by 
\[
\Mx{Q}_{t+1} = \Mx{Q}_t + \eta \sum_{n=1}^N r_n \Mx{X}_n \Mx{QM},
\]
which potentially gets $\Mx{Q}$ back to a full-rank path (depending on $\Mx{M}$) and this process starts all over again. This observation in part provides intuition on why empirically spurious stationary points of $\llq$ are close to the global minimum (in terms of the objective value).

\textbf{Remark 2:} (Choice of the step-size) The proof of the second part of Theorem \ref{thm:main-alpha} does not depend on the step-size of gradient descent as long as the value of the objective function at the critical point is less than the initialization which is a very mild requirement. As for the proof of the first part, as discussed in Lemma \ref{lem:step_size} and Corollary \ref{cor:step_size} a big step-size can be mitigated by a scaling in the initialization of network weights. This is due to the fact that stationary points are flat in some directions due to the symmetry with respect to the scaling ($\Mx{Q} \rightarrow \Mx{Q}/\beta, \Mx{M} \rightarrow \beta^2 \Mx{M}$). Although our results might appear asymptotic, this observation suggests a linear rate of convergence with a small enough step-size, since it guarantees the Hessian eigenvalues are nonzero at a stationary point which is not a global minimum.

\textbf{Remark 3:} (Choice of the norm) The proof of Theorem \ref{thm:main-alpha} holds when $\Mx{I}$ is replaced by any positive-definite matrix $\Mx{D}$. In particular for a stationary point we can assume without loss of generality that $\sum_{n=1}^N r_n \Mx{X}_n \succeq \Mx{0}$. If $\lambda_{\min}$ is the smallest eigenvalue of $\Mx{D}$ we have
\[
0 = \frac{\partial}{\partial \alpha} \mathcal{L}(\alpha, \Mx{M},\Mx{Q}) = \sum_{n=1}^N r_n \Mx{X}_n \bullet \Mx{D} \geq \lambda_{\min} \trace(\sum_{n=1}^N r_n \Mx{X}_n) \geq 0,
\]
with equality if and only if $\sum_{n=1}^N r_n \Mx{X}_n = \Mx{0}$.

\textbf{Remark 4:} Part 3 of Theorem \ref{thm:main-alpha} can be interpreted as follows. By Lemma \ref{lem:qq_lb} and by continuity of the penalty term $\left\|\Mx{QQ}^T-\Mx{I} \right\|_{\sf Fro}^2$ we conclude that in a neighborhood of $\Mx{Q}$ the penalty term is dominant and the landscape of $\mathcal{F}$ is similar to that of the penalty, which by Lemma \ref{lem:orthogonal_penalty} has a negative curvature direction.

\begin{corollary} \label{cor:gd_svd}
Let $\mathcal{A}= \{\Mx{A} \in \R^{d \times d} | \Mx{A} \text{ is a global minimum of (\ref{eq:cvx})} \}$. Then for $\alpha = 0$, gradient descent converges to an eigenvalue decomposition of the form $\Mx{A} = \Mx{Q}\Mx{\Lambda}\Mx{Q}^T$ for some $\Mx{A} \in \mathcal{A}$. In particular, if (\ref{eq:cvx}) is strictly convex $\mathcal{A}$ is a singleton and $(\V{\lambda},\Mx{Q})$ is unique.
\end{corollary}

\begin{corollary} Let $\Mx{A} \in \R^{d \times d}$ be a symmetric matrix and $\gamma = \|\Mx{A}\|_{\sf Fro}^2$. Then starting from $(\V{\lambda}_0 = \V{0},\Mx{Q}_0 = \Mx{I}_d)$ the solution to
\[
\minimize_{\V{\lambda},\Mx{Q}} \|\Mx{A}-\Mx{Q \Lambda Q}^T\|_{\sf Fro}^2 + \gamma \|\Mx{QQ}^T -\Mx{I}\|_{\sf Fro}^2,
\]
is a global minimum (achieving 0 value of the objective) and is an eigenvalue decomposition of $\Mx{A}$.
\end{corollary}
\begin{proof}
This is a special case of Corollary \ref{cor:gd_svd} with $\Mx{X}_n = \Mx{X}_{ij} = \frac{(\V{e}_i+\V{e}_j)(\V{e}_i+\V{e}_j)^T-(\V{e}_i-\V{e}_j)(\V{e}_i-\V{e}_j)^T}{4}$, $y_n = A_{ij}$.
\end{proof}

\noindent In light of Remark 2 above, the eigenvalue decomposition algorithm obtained using gradient descent is stable and of the same computational complexoity matrix multiplication ($\mathcal{O}(n^\omega)$ for some $2 \leq \omega < 2.373$) which matches the previously known results on stability and computational complexity of eigenvalue decomposition (c.f. \cite{demmel2007fast}).

%% file: app_proof_multivariate.tex
\section{Proof of the Main Theorem: The Multivariate Case} \label{sec:multivariate}
\renewcommand{\theequation}{\Alph{section}.\arabic{equation}}
\renewcommand{\thefigure}{\Alph{section}.\arabic{figure}}

In this section we will generalize Theorem \ref{thm:main_text:main} to the case where there are several multiple outputs in the network. Understanding the multivariate output case is crucial to understanding deep QL networks. In the multivariate case the hidden layer neurons are assumed to be shared between the outputs, which adds extra difficulty to the analysis. Let $M$ be the dimension of the output. The objective function in the multivariate case can be expressed as
\begin{align} \label{eq:multivariate_mseloss}
\mathcal{L}_k \left( {\Mx{\Lambda}_1, \ldots, \Mx{\Lambda}_M,\Mx{Q}} \right) := \frac{1}{MN}\sum_{m, n = 1}^{M,N} \left(y_{mn}-\left( \Mx{Q}\Mx{\Lambda}_m\Mx{Q}^T+ \alpha_m \Mx{I} \right) \bullet \Mx{X}_n \right)^2, 
\end{align}
for $\Mx{\Lambda}_i = \diag(\V{\lambda}_i) \in \mathbb{R}^{k \times k}$, $\Mx{Q} \in \mathbb{R}^{d \times k}$ and we wish to characterize the landscape of the optimization problem
\begin{equation} \label{eq:multivariate_objective}
\minimize_{\Mx{\Lambda}_1, \ldots, \Mx{\Lambda}_M,\Mx{Q}} \mathcal{L}_k \left( {\Mx{\Lambda}_1, \ldots, \Mx{\Lambda}_M,\Mx{Q}} \right).
\end{equation}

\begin{theorem} [Single-Layer Quadratic Linear Networks with Multivariate Outputs] \label{thm:main_multivariate} Let $k \geq Md$ and $\Mx{\Lambda}_i, \Mx{Q}$ be initialized from an arbitrary absolutely continuous distribution. Then  with probability 1 gradient descent will 
go to a global minimum of (\ref{eq:multivariate_objective}).
\end{theorem}

\begin{proof} The key to proving Theorem \ref{thm:main_multivariate} is to decompose the multivariate case into $M$ univariate problems.
Let $(\Mx{\Lambda}_1, \ldots, \Mx{\Lambda}_M, \Mx{Q})$ be a local minimum to (\ref{eq:multivariate_objective}). For $m = 1, ..., M$ define
\[
\Mx{B}_{m} = \sum_{i \in I_m^c} \lambda_{mi} \V{q}_i \V{q}_i^T + \sum_{m'=1, m' \neq m}^M \alpha_{m'} \Mx{I} ,
\]
where $I_m = \{(m-1)d+1,\cdots, md\}$ and $I_m^c = \{1,2, \cdots, (m-1)d\} \bigcup \{ md+1, \cdots, k \}$. Now for $m=1, \cdots, M$ consider the optimization problem

\begin{equation} \label{eq:Pm}
\tag{$P_m$}    \minimize_{\lambda_{mi}, \V{q}_i, i \in I_m^c} \frac{1}{N} \sum_{n=1}^N \left( y_{nm} - \Mx{B}_m \bullet \Mx{X}_n - \left( \sum_{i \in I_m} \lambda_{mi} \V{q}_i \V{q}_i^T +\alpha_m \Mx{I}  \right) \bullet \Mx{X}_n  \right)^2.
\end{equation}

By Lemma \ref{lem:k_equals_d} for each $m$ the global minimum to (\ref{eq:Pm}) is the same as when $\Mx{B}_m = 0$. Moreover by Theorem \ref{thm:main_text:main}, and using random initialization with probability 1 any solution of (\ref{eq:Pm}) found by gradient descent is a global minimum. Moreover, given $\Mx{B}_m$'s, (\ref{eq:Pm})'s are independent in the sense that the optimization involves different set of parameters. Since each subproblem achieves its global minimum, (\ref{eq:multivariate_objective}) also achieves its global minimum.
\end{proof}


\noindent Note that the number of neurons $k = Md$ is highly inefficient for $M \gg d$. This is due to the fact that there exists a $\binom{d+1}{2}$-dimensional basis for $d \times d$-dimensional matrices. In particular the following somewhat trivial result suggests that for $k \geq \binom{d+1}{2}$ training quadratic neural networks is a trivial task and only requires solving a least squares problem..
\begin{lemma} \label{lem:maxK} The maximum number of required neurons for a two-layer QNN with $d$-dimensional input is equal to $k = \binom{d+1}{2}$. Moreover in this case training $\Mx{Q}$ is unnecessary and is only needed on $\Lambda_i$'s which is a linear regression and hence convex.
\end{lemma}
\begin{proof}
Let $\V{e}_1, \cdots, \V{e}_d$ denote the unit vectors in $\mathbb{R}^d$. Then the symmetric rank-1 matrices $\Mx{E}_{ij} = (\V{e}_i+\V{e}_j)(\V{e}_i+\V{e}_j)^T$, $i,j = 1,2, \cdots d$, $j\geq i$, form a basis for symmetric matrices in $\mathbb{R}^{d \times d}$. Fixing $\V{q}_{ij} =\V{e}_i + \V{e}_j$ as the input weights the coefficients the claim follows from the fact that any symmetric matrix $\Mx{A}$ has a unique representation of the form $\Mx{A} = \Mx{Q \Lambda}_m \Mx{Q}$. In particular choose $\Mx{\Lambda}_m$ to be the representation for the global solutions given by the linear regressions problems
\begin{equation} \label{eq:cvx2}
\minimize_{\Mx{A}_m} \frac{1}{N} \sum_{n = 1}^N \left(y_{mn} - \Mx{A}_m \bullet \Mx{X}_n \right)^2.
\end{equation}
\end{proof}

\noindent Finally, similar to Theorem \ref{thm:main-alpha} we can use an orthogonality penalty on $\Mx{Q}$ to achieve global optimality.

\begin{theorem} \label{thm:multi_penalized} Let $\gamma > \frac{1}{MN}\sum_{n, m=1}^{N,M} y_{nm}^2$ and define
$I_m = \{(m-1)d+1,\cdots, md\}$.
Then with the initialization $(\{\V{\lambda}_m = \Mx{0}\}_{m=1}^M,\Mx{Q}_{I_m}=\Mx{I})$ all critical points of
\begin{equation} \label{eq:multi_penalized}
\minimize_{\{\V{\lambda}_m \in \R^{Md}\}_{m=1}^M,\Mx{Q} \in \R^{d \times Md}}
\frac{1}{MN}\sum_{n,m=1}^{N,M} ( y_{nm}-\Mx{Q \Lambda}_m\Mx{Q}^T \bullet \Mx{X}_n)^2 + \gamma \sum_{m=1}^M  \|\Mx{Q}_{I_m}\Mx{Q}_{I_m}^T-\Mx{I}\|_{\sf Fro}^2    
\end{equation}
are global minima.
\end{theorem}
\begin{proof}
The proof is in effect a combination of the technique used in Theorems \ref{thm:main-alpha} and \ref{thm:main_multivariate}.
Similar to the proof of Theorem \ref{thm:main_multivariate}, (\ref{eq:multi_penalized}) can be broken into $M$ independent univariate problems and similar to the proof of Theorem \ref{thm:main-alpha}, the penalty term ensures $\rank(\Mx{Q}_{I_m}) = d$ for all $m$. Therefore, by Theorem \ref{thm:main-alpha} all stationary points of (\ref{eq:multi_penalized}) are global.
\end{proof}

%% file: app_proof_deep.tex
\label{app:deep}
\section{Proof of the Main Theorem on Deep Quadratic-Linear Networks}
\renewcommand{\theequation}{\Alph{section}.\arabic{equation}}
\renewcommand{\thefigure}{\Alph{section}.\arabic{figure}}

\subsection{Notations}
We first define the notations used in this section. We assume the deep network consists of $L$ quadratic followed by linear (QL) layers (see Figure \ref{fig:deep_notation}b). In general we denote the parameters related to the $l$-th layer by superscripts.
\begin{itemize}
    \item We denote the output of the $l$-th layer by $x_n^{(l)} \in \R^{h_l}$. 
    \item For simplicity we use denote the input to the network by $\V{x}_n = \V{x}_n^{(0)} \in \R^{h_0}$ and the output is given by ${x}_n^{(L)} \in \R$.
    \item The parameters of the $l$-th layer are given by $\left( \left\{ \Mx{\Lambda}_i^{(l)} \in \R^{m_l \times m_l} \right\}_{i=1}^{h_l}, \Mx{Q}^{(l)} \in \R^{h_{l-1} \times m_l} \right)$ and we define 
    \[
    \Mx{A}_i^{(l)}:=\Mx{Q}^{(l)} \Mx{\Lambda}_i^{(l)} {\Mx{Q}^{(l)}}^T.\]
    By definition of the network we have
    \[
    x_{ni}^{(l)}=x_n^{(l)}(i) = \Mx{A}_i^{(l)} \bullet \Mx{X}_n^{(l-1)}. 
    \]
    \item We define the set of all parameters of the network by 
    \[
     \left(\lambda, Q \right)_i^j: = \left\{ \left\{ \Mx{\Lambda}_{i}^{(l)} \right\}_{i=1}^{h_l}, \Mx{Q}^{(l)}  \right\}_{i=1}^j,
    \] 
    \[
    \left(\lambda, Q \right)=\left(\lambda, Q \right)_1^L.
    \]
    \item We define 
    \[
    \widetilde{\Mx{Q}}^{(l)}:=  \left[\vectorize{\left(\Mx{A}_1^{(l)}\right)},\cdots,\vectorize{\left(\Mx{A}_{h_l}^{(l)}\right)} \right], \quad l=1,\cdots,L-1,
    \]
    and
    \[\widetilde{\Mx{Q}}_L =\Mx{Q}_L.\]
    \item For a rank-1 symmetric tensor $\V{x}^{\otimes k} \otimes \V{x}^{\otimes k}$ we denote its matricization to be the symmetric matrix
    \[
    \matricize{\left( \V{x}^{\otimes k} \otimes \V{x}^{\otimes k} \right)}: =\vectorize{\left(\V{x}^{\otimes k}\right)} \vectorize{\left(\V{x}^{\otimes k}\right)}^T.
    \]
    
\end{itemize}

\subsection{Problem Formulation and Main Result}
For theoretical purposes of this paper we are interested in the objective function
\begin{equation}
    \mathcal{F}_{\sf{deep}}:= \frac{1}{N}\sum_{n=1}^N \left( {y}_n - {x}_n^{(L)}\right)^2 + \gamma \sum_{l=1}^{L} \left\|  \widetilde{\Mx{Q}}^{(l)}\widetilde{\Mx{Q}}^{(l)^T}-\Mx{I}_{h_l}\right\|_{\sf Fro}^2,
\end{equation}
and would like to characterize the landscape of the optimization problem
\begin{equation} \label{eq:deep_objective}
    \minimize_{ \left(\lambda, Q \right)} \mathcal{F}_{\sf{deep}}.
\end{equation}
Throughout the proofs we define
\[
{r}_n := {y}_n - {x}_n^{(L)}.
\]

\noindent The following Lemma is a simple extension of Lemma \ref{lem:KKT}.
\begin{lemma}[Optimality Conditions]
$(\lambda,Q)$ is the global minimum to $\mathcal{L}_{\sf deep}$ if and only if
\begin{equation}
        \sum_{n=1}^N r_n {\V{x}_n^{\otimes 2^L}} = \Mx{0}.
\end{equation}
\end{lemma}
\begin{proof}
Proof is similar to lemma \ref{lem:KKT} and is left out due to simplicity.
\end{proof}

\begin{theorem}[Main Theorem on Deep Quadratic-Linear Networks] \label{thm:main_deep} 
Consider an instance of the overparameterized deep QL network of depth-$L$ and the parameters
    \begin{itemize}
    \item $h_0= d$, the input dimension, and $h_l = d^{2^{L-l}}$ for $l=1,\cdots,L$,
    \item $m_1 = h_{l-1} h_l$ for $l=2,\cdots,L$.
\end{itemize}
Then for sufficiently large $\gamma$ any stationary point of $\mathcal{F}_{\sf deep}$ either has a negatuve curvature direction or is a global minimum.
\end{theorem}
\begin{proof} We prove the assertion by induction on $L$. The induction base $L=1$ follows from Theorem \ref{thm:multi_penalized}. Let $(\lambda,Q)$ be a stationary point of the $L$ layer network and let $\Mx{M} \in \R^{d \otimes 2^{L-1}}$ be the solution to 
\begin{equation}
    \minimize_{\Mx{M}} \frac{1}{N} \sum_{n=1}^{N} \left(y_n - \Mx{M}\bullet \V{x}^{(1)^{\otimes 2^{L-1}}}_n \right)^2.
\end{equation}
By induction hypothesis $(\lambda,Q)_{2}^L$ are a reparameterization of $\Mx{M}$. In particular for such $\Mx{M}$ Lemma \ref{lem:lr_stationary} holds. Since $\V{x}_{ni}^{(1)} = \Mx{A}_i^{(1)} \bullet \Mx{X}_n$.
we have
\begin{align}
    x_{n}^{(L)} & =  \Mx{M} \bullet \V{x}^{(1)^{\otimes 2^{L-1}}}_n = \sum_{i,j} M_{ij} \left( \Mx{A}_i^{(1)} \otimes \Mx{A}_j^{(1)} \right) \bullet  \V{x}^{{\otimes 2^{L}}}_n = \widetilde{\Mx{Q}}^{(1)}\Mx{M}\widetilde{\Mx{Q}}^{(1)^T} \bullet  \V{x}^{{\otimes 2^{L}}}_n.
\end{align}
Finally, similar to the proof of Theorem \ref{thm:main-alpha}, a sufficiently large $\gamma$ ensures the conditions of Theorem \ref{thm:main-alpha} are met and $\rank(\widetilde{\Mx{Q}})^{(1)} = d^{2^{L-1}}$. Therefore by Lemma \ref{lem:lr_stationary}, $(\lambda,Q)$ must be a global minimum.
\end{proof}

\textbf{Remark 1:} The proof technique used for the deep QL network solely relies on two facts:
\begin{enumerate}
    \item The depth-1 ($L=1$) network does not have any spurious stationary points and
    \item Given enough neurons a composition of two depth-1 networks $(L=2)$ can be re-written as a depth-1 network. In particular, the case $L=2$ is a regression against quartic features ($\V{x}_n^{\otimes 4}$) which can  interpreted as a quadratic regression against quadratic features ($\Mx{X}_n^{\otimes 2}$).
\end{enumerate}
Therefore, our theory suggests that our results hold for other activation functions such that these two conditions are met. As an example, higher order polynomial activation functions satisfy condition 2 in the same way as quadratics. 

\textbf{Remark 2:} In our proof we used the fact that every quadratic layer is followed by a linear layer. As argued in the remarks following the proof of Theorem \ref{thm:main-alpha} linear layers can be thought of eigenvalues of the input-output transition matrices $\Mx{A}_i^{(l)}$. Hence they allow for arbitrary input-output transition matrices in the overall network. Moreover, the addition of the extra linear layer adds minimal complexity to the network as \ref{eq:deep_objective} is convex in each of the linear weights separately. However, the linear layers in the intermediate layers ($l=1, \cdots, L-1$) can be absorbed into the quadratic weights of the subsequent letters and there is no need to optimize over them.

\textbf{Remark 3:} Although from a theoretical point of view our result shows absence of \textit{any} spurious critical points, the number of neurons required by the network to guarantee such a result grows superexponentially in the number of layers which is prohibitive in real-world applications. Such a large overparameterization is needed not only because of global optimality of a single layer, but also because neurons from any layer should help other layers to resolve their local minima. However, we have observed that in practice, global minimum is achieved using far fewer neurons even in the absence of any penalty terms on the network. Moreover, the expressive power of such a network also grows exponentially fast as it can represent polynomials of maximum degree $2^L$.

\textbf{Remark 4:} Note that solving for a depth-$L$ network with this many neurons is equivalent to regressing against $2^{L}$-th order features in the input which is a linear regression problem. In particular, we can conclude that the solution to (\ref{eq:deep_objective}) is equivalent to the global minimum to the convex problem
\[
 \minimize_{\Mx{A}} \frac{1}{N}\sum_{n=1}^{N} \left(y_n - \Mx{A}\bullet \V{x}^{{\otimes 2^{L}}}_n \right)^2.
\]
Writing $\Mx{A} = \sum_{i = 1}^{d^{2^{L}}} \lambda_i \V{q}_i^{\otimes 2^{L}}$, we can see the equivalence to a single layer network with $2^{L}$-th order activation functions.

\textbf{Remark 5:} Based on the intuition from remark 4, we can see that achieving the global minimum in general should be possible using far fewer neurons as the maximal rank of a symmetric tensor in only bounded above by the number $d^{2^L}$ (as given by a basis obtained by the outer product of canonical euclidean bases). Finding the maximal rank of a symmetric tensor in general is an open problem.

%% file: app_proof_tensor.tex
\label{app:poly}
\section{Some Results on Higher Order Activation Functions} \label{sec:polynomial}
\renewcommand{\theequation}{\Alph{section}.\arabic{equation}}
\renewcommand{\thefigure}{\Alph{section}.\arabic{figure}}

\subsection{Notations}
We first define the notations related to tensors and higher order activation functions in this section.

\begin{itemize}
    \item Tensor $\Mx{A} \in \R^{d_1 \times d_2 \times \cdots d_p}$ is said to be rank-$r$, denoted by $\rank(\Mx{A}) = r$ if $r$ is the smallest integer such that
    \[
    \Mx{A} = \sum_{i = 1}^r \lambda_i \V{a}_{1_i} \otimes \V{a}_{2_i} \otimes \cdots \otimes \V{a}_{p_i}.
    \]
    \item A tensor $\Mx{A} \in \R^{d^{\otimes p}}$ is said to be symmetric if $A_{\pi_{i_1,i_2,\cdots,i_p}} = A_{i_1,i_2,\cdots,i_p}$, where $\pi$ is an arbitrary permutation of the indices
    \item For a symmetric tensor $\Mx{A} \in \R^{d {\otimes p}}$ we denote its symmetric rank by $\symrank(\Mx{A})$ which is the the smallest $r$ such that
    \[
    \Mx{A} = \sum_{i=1}^r \lambda_i \V{a}_i^{\otimes p}.
    \]
\end{itemize}

\subsection{Problem Formulation and Main Result}
In this section we state and prove the main Theoretical results of this paper on neural networks with higher order polynomial activation functions. In particular similar to our results for quadratic networks, we will show deep neural networks with polynomial activation functions which of degree $p = 2^m$ do not have local minima. 

For theoretical purposes of this section we consider deep networks with structure shown in Figure \ref{fig:PLnet}. These networks consist of layers with polynomial activation functions followed by linear layers (hence the name PL networks).

\begin{figure}[H]
\begin{center}
\noindent
\includegraphics[width=\textwidth]{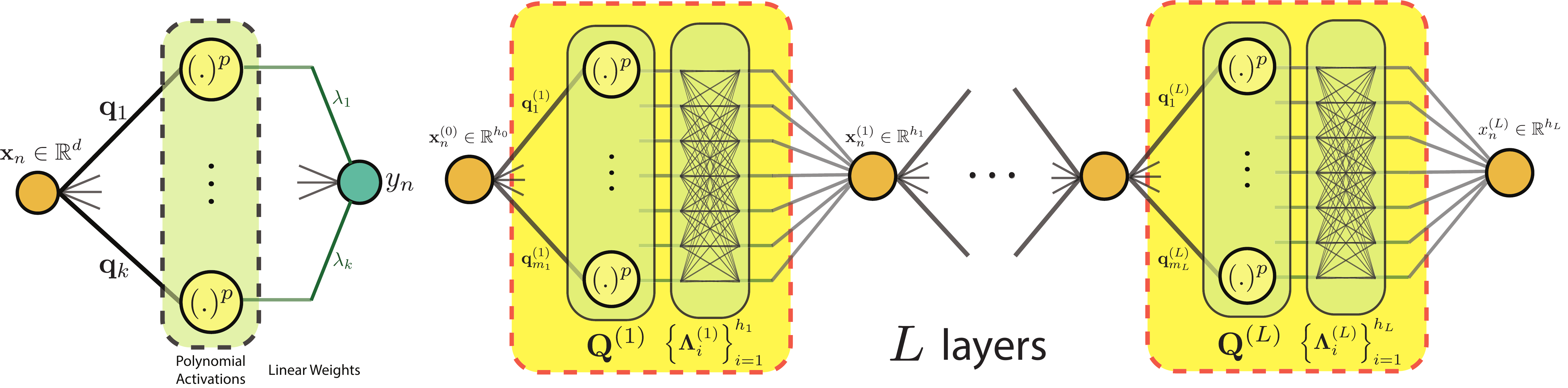}
\caption{Two-layer and deep polynomial-linear neural networks with scalar output}
\label{fig:PLnet}
\end{center}
\end{figure}

For a two-layer PL-network with $k$ neurons and scalar outputs we would like to study the landscape of the optimization problem
\begin{align} \label{eq:mseloss_PL}
    \minimize_{\V{\lambda},\Mx{Q}} \mathcal{F}(\V{\lambda},\Mx{Q}):&= \frac{1}{N}\sum_{n=1}^N \Big(y_n - \sum_{i=1}^k \lambda_i \V{q}_i^{\otimes p} \bullet \Mx{X}_n \Big)^2 + \gamma \left\|\left(\Mx{Q}^T\Mx{Q}\right)^{\odot p}-\Mx{I}_k\right\|_{\sf Fro}^2 \\
    & = \frac{1}{N}\sum_{n=1}^N \Big(y_n - \sum_{i=1}^k \lambda_i \V{q}_i^{\otimes p} \bullet \Mx{X}_n \Big)^2 + \gamma \left\|\Mx{R}^T\Mx{R}-\Mx{I}_k\right\|_{\sf Fro}^2,
\end{align}
where 
\[
\Mx{R} := \left[ \vectorize \left(\V{q}^{\otimes p}_1\right),\cdots,\vectorize \left(\V{q}^{\otimes p}_k\right) \right] \in \R^{d^p \times \binom{d+p-1}{d-1 }},
\]
\noindent $\Mx{Q}:=[\vectorize(\V{q}_1),\cdots,\vectorize(\V{q}_k)] \in \R^{d \times k} $ and $\Mx{X}_n = \V{x}_n^{\otimes p}$. Note that $\V{q}_i^{\otimes p} \bullet \Mx{X}_n = \left(\V{q}_i^T \V{x}_n\right)^p$. The following Theorem is a generalization of Theorem \ref{thm:main-alpha}.

\begin{theorem}[Main Theorem on PL networks] \label{thm:PL} Let $\gamma > \frac{1}{N}\sum_{n=1}^N y_n^2$ and $k = \binom{d+p-1}{d-1 }$. Then using the initialization $(\V{\lambda}_0 = \V{0})$ and $\V{q}_{i_1i_2\cdots i_p } : = \sum_{j=1}^p \V{e}_{i_j}$ for $i_1 \leq i_2 \leq \cdots \leq i_{p-1} \leq i_p$, gradient descent with arbitrary step-size will converge to a global minimum of (\ref{eq:mseloss_PL}).
\end{theorem}
\begin{proof} It is easy to see that $\V{q}_{i_1i_2\cdots i_p }$'s are linearly independent. Hence $\rank(\Mx{R}_0) = \binom{d+p-1}{d-1 }$. Using the same idea as in the proof of Theorem \ref{thm:main-alpha} we conclude $\rank(\Mx{R}) = \binom{d+p-1}{d-1 }$. Let $\Mx{A}$ be a global solution to the convex problem
\begin{equation} \label{eq:mseloss_PL_cvx}
    \minimize_{\Mx{A}} \frac{1}{N}\sum_{n=1}^N \Big(y_n - \Mx{A} \bullet \Mx{X}_n \Big)^2.
\end{equation}
Since $\Mx{R}$ is full-rank there exist scalars $\lambda_{1}, \cdots, \lambda_k$ such that
\[
\Mx{A} = \sum_{i=1}^k \lambda_i \V{q}_i^{\otimes k}.
\]
The assertion follows from the convexity of (\ref{eq:mseloss_PL}) in $\V{\lambda}$ given $\Mx{Q}$.
\end{proof}

\textbf{Remark 1:} There are several differences between Theorem \ref{thm:PL} and its counterpart for QL networks (Theorem \ref{thm:main-alpha}):
\begin{enumerate}
    \item The minimum number of neurons required for a general PL network is $k = \mathcal{O}(d^p)$, whereas for QL networks $k = d$ neurons is sufficient despite the fact that $p = 2$. In particular, the best lower bound we can hope for is $k = r_{\max}(d,p)$, where $r_{\max}(p)$ is the maximum symmetric rank of a $p$-way tensor (e.g. see Lemma \ref{lem:k_equals_d}). However, achieving such a lower bound seems to a challenging task for several reasons. First and foremost, determining the maximal rank of a $p$-way tensor is not theoretically known and is an open problem. Second, the Eckart-Young Theorem does not hold for $p>2$ and Lemma \ref{cor:continuous_eig} does not readily generalize to $p>2$. Therefore, finding an optimal decomposition of higher order tensors and design of a penalty that vanishes on full-rank tensors and is lower bounded on low-rank tensors is challenging if possible at all.
    
    \item Theorem \ref{thm:PL} requires precise initialization, whereas global optimality was achieved independent of the initialization in Theorem \ref{thm:main-alpha}. The main difficulty in removing such dependence on the initialization stems from the difficulty in designing a suitable penalty explained above. In particular the penalty used in (\ref{eq:mseloss_PL}) could potentially have local minima for $p > 2$. Therefore the proof of Theorem \ref{thm:main-alpha}-3 does not readily generalize to $p > 2$.
\end{enumerate}

\textbf{Remark 2:} In the proof of Theorem \ref{thm:PL} we have used the upper bound $r_{\max}(d,p) \leq \binom{d+p-1}{d-1}$ and the somewhat trivial standard Euclidean bases. In particular, given $k= \binom{d+p-1}{d-1} $ neurons and the initialization using standard Euclidean bases, there is no need to update $\Mx{Q}$ and updating $\V{\lambda}$ (which is convex) via gradient descent or least squares suffices to achieve global optimality. 

\textbf{Remark 3:} Theorem \ref{thm:PL} can be generalized to deep PL networks similar to Theorem \ref{thm:main_deep} via correct initialization. It is also worth noting that a deep polynomial network is equivalent to a two-layer PL network with higher order activation functions and hence has no particular advantage over Theorem \ref{thm:PL}.

%% file: main.bbl
\begin{thebibliography}{10}

\bibitem{zhang2016understanding}
Chiyuan Zhang, Samy Bengio, Moritz Hardt, Benjamin Recht, and Oriol Vinyals.
\newblock Understanding deep learning requires rethinking generalization.
\newblock {\em arXiv preprint arXiv:1611.03530}, 2016.

\bibitem{blum1989training}
Avrim Blum and Ronald~L Rivest.
\newblock Training a 3-node neural network is np-complete.
\newblock In {\em Advances in neural information processing systems}, pages
  494--501, 1989.

\bibitem{burer2003nonlinear}
Samuel Burer and Renato~DC Monteiro.
\newblock A nonlinear programming algorithm for solving semidefinite programs
  via low-rank factorization.
\newblock {\em Mathematical Programming}, 95(2):329--357, 2003.

\bibitem{burer2005local}
Samuel Burer and Renato~DC Monteiro.
\newblock Local minima and convergence in low-rank semidefinite programming.
\newblock {\em Mathematical Programming}, 103(3):427--444, 2005.

\bibitem{JS19}
Mahdi Soltanolkotabi, Adel Javanmard, and Jason~D Lee.
\newblock Theoretical insights into the optimization landscape of
  over-parameterized shallow neural networks.
\newblock {\em IEEE Transactions on Information Theory}, 65(2):742--769, 2019.

\bibitem{ge2016matrix}
Rong Ge, Jason~D Lee, and Tengyu Ma.
\newblock Matrix completion has no spurious local minimum.
\newblock In {\em Advances in Neural Information Processing Systems}, pages
  2973--2981, 2016.

\bibitem{boumal2016non}
Nicolas Boumal, Vlad Voroninski, and Afonso Bandeira.
\newblock The non-convex burer-monteiro approach works on smooth semidefinite
  programs.
\newblock In {\em Advances in Neural Information Processing Systems}, pages
  2757--2765, 2016.

\bibitem{ge2017no}
Rong Ge, Chi Jin, and Yi~Zheng.
\newblock No spurious local minima in nonconvex low rank problems: A unified
  geometric analysis.
\newblock In {\em Proceedings of the 34th International Conference on Machine
  Learning-Volume 70}, pages 1233--1242. JMLR. org, 2017.

\bibitem{de2014global}
Christopher De~Sa, Kunle Olukotun, and Christopher R{\'e}.
\newblock Global convergence of stochastic gradient descent for some non-convex
  matrix problems.
\newblock {\em arXiv preprint arXiv:1411.1134}, 2014.

\bibitem{bhojanapalli2016global}
Srinadh Bhojanapalli, Behnam Neyshabur, and Nati Srebro.
\newblock Global optimality of local search for low rank matrix recovery.
\newblock In {\em Advances in Neural Information Processing Systems}, pages
  3873--3881, 2016.

\bibitem{cheng2018polynomial}
Xi~Cheng, Bohdan Khomtchouk, Norman Matloff, and Pete Mohanty.
\newblock Polynomial regression as an alternative to neural nets.
\newblock {\em arXiv preprint arXiv:1806.06850}, 2018.

\bibitem{hornik91}
Kurt Hornik.
\newblock Approximation capabilities of multilayer feedforward networks.
\newblock {\em Neural networks}, 4(2):251--257, 1991.

\bibitem{lu2017expressive}
Zhou Lu, Hongming Pu, Feicheng Wang, Zhiqiang Hu, and Liwei Wang.
\newblock The expressive power of neural networks: A view from the width.
\newblock In {\em Advances in Neural Information Processing Systems}, pages
  6231--6239, 2017.

\bibitem{fan18}
Fenglei Fan and Ge~Wang.
\newblock Universal approximation with quadratic deep networks.
\newblock {\em arXiv preprint arXiv:1808.00098}, 2018.

\bibitem{kileel2019expressive}
Joe Kileel, Matthew Trager, and Joan Bruna.
\newblock On the expressive power of deep polynomial neural networks.
\newblock In {\em Advances in Neural Information Processing Systems}, pages
  10310--10319, 2019.

\bibitem{rust2005spatiotemporal}
Nicole~C Rust, Odelia Schwartz, J~Anthony Movshon, and Eero~P Simoncelli.
\newblock Spatiotemporal elements of macaque v1 receptive fields.
\newblock {\em Neuron}, 46(6):945--956, 2005.

\bibitem{tian2017analytical}
Yuandong Tian.
\newblock An analytical formula of population gradient for two-layered relu
  network and its applications in convergence and critical point analysis.
\newblock In {\em Proceedings of the 34th International Conference on Machine
  Learning-Volume 70}, pages 3404--3413. JMLR. org, 2017.

\bibitem{brutzkus2017globally}
Alon Brutzkus and Amir Globerson.
\newblock Globally optimal gradient descent for a convnet with gaussian inputs.
\newblock In {\em Proceedings of the 34th International Conference on Machine
  Learning-Volume 70}, pages 605--614. JMLR. org, 2017.

\bibitem{li2017convergence}
Yuanzhi Li and Yang Yuan.
\newblock Convergence analysis of two-layer neural networks with relu
  activation.
\newblock In {\em Advances in Neural Information Processing Systems}, pages
  597--607, 2017.

\bibitem{janzamin2015beating}
Majid Janzamin, Hanie Sedghi, and Anima Anandkumar.
\newblock Beating the perils of non-convexity: Guaranteed training of neural
  networks using tensor methods.
\newblock {\em arXiv preprint arXiv:1506.08473}, 2015.

\bibitem{zhong2017recovery}
Kai Zhong, Zhao Song, Prateek Jain, Peter~L Bartlett, and Inderjit~S Dhillon.
\newblock Recovery guarantees for one-hidden-layer neural networks.
\newblock In {\em Proceedings of the 34th International Conference on Machine
  Learning-Volume 70}, pages 4140--4149. JMLR. org, 2017.

\bibitem{panigrahy2017convergence}
Rina Panigrahy, Ali Rahimi, Sushant Sachdeva, and Qiuyi Zhang.
\newblock Convergence results for neural networks via electrodynamics.
\newblock {\em arXiv preprint arXiv:1702.00458}, 2017.

\bibitem{javanmard2019analysis}
Adel Javanmard, Marco Mondelli, and Andrea Montanari.
\newblock Analysis of a two-layer neural network via displacement convexity.
\newblock {\em arXiv preprint arXiv:1901.01375}, 2019.

\bibitem{mondelli2018connection}
Marco Mondelli and Andrea Montanari.
\newblock On the connection between learning two-layers neural networks and
  tensor decomposition.
\newblock {\em arXiv preprint arXiv:1802.07301}, 2018.

\bibitem{livni2014computational}
Roi Livni, Shai Shalev-Shwartz, and Ohad Shamir.
\newblock On the computational efficiency of training neural networks.
\newblock In {\em Advances in neural information processing systems}, pages
  855--863, 2014.

\bibitem{algReg18}
Yuanzhi Li, Tengyu Ma, and Hongyang Zhang.
\newblock Algorithmic regularization in over-parameterized matrix sensing and
  neural networks with quadratic activations.
\newblock In {\em Conference On Learning Theory}, pages 2--47, 2018.

\bibitem{du2018power}
Simon Du and Jason Lee.
\newblock On the power of over-parametrization in neural networks with
  quadratic activation.
\newblock In {\em International Conference on Machine Learning}, pages
  1329--1338, 2018.

\bibitem{venturi2019spurious}
Luca Venturi, Afonso~S Bandeira, and Joan Bruna.
\newblock Spurious valleys in one-hidden-layer neural network optimization
  landscapes.
\newblock {\em Journal of Machine Learning Research}, 20(133):1--34, 2019.

\bibitem{soudry2016no}
Daniel Soudry and Yair Carmon.
\newblock No bad local minima: Data independent training error guarantees for
  multilayer neural networks.
\newblock {\em arXiv preprint arXiv:1605.08361}, 2016.

\bibitem{nguyen2017loss}
Quynh Nguyen and Matthias Hein.
\newblock The loss surface of deep and wide neural networks.
\newblock {\em arXiv preprint arXiv:1704.08045}, 2017.

\bibitem{choromanska2015loss}
Anna Choromanska, Mikael Henaff, Michael Mathieu, G{\'e}rard~Ben Arous, and
  Yann LeCun.
\newblock The loss surfaces of multilayer networks.
\newblock In {\em Artificial intelligence and statistics}, pages 192--204,
  2015.

\bibitem{dauphin2014identifying}
Yann~N Dauphin, Razvan Pascanu, Caglar Gulcehre, Kyunghyun Cho, Surya Ganguli,
  and Yoshua Bengio.
\newblock Identifying and attacking the saddle point problem in
  high-dimensional non-convex optimization.
\newblock In {\em Advances in neural information processing systems}, pages
  2933--2941, 2014.

\bibitem{pennington2017geometry}
Jeffrey Pennington and Yasaman Bahri.
\newblock Geometry of neural network loss surfaces via random matrix theory.
\newblock In {\em Proceedings of the 34th International Conference on Machine
  Learning-Volume 70}, pages 2798--2806. JMLR. org, 2017.

\bibitem{bahri2020statistical}
Yasaman Bahri, Jonathan Kadmon, Jeffrey Pennington, Sam~S Schoenholz, Jascha
  Sohl-Dickstein, and Surya Ganguli.
\newblock Statistical mechanics of deep learning.
\newblock {\em Annual Review of Condensed Matter Physics}, 2020.

\bibitem{kawaguchi2019depth}
Kenji Kawaguchi and Yoshua Bengio.
\newblock Depth with nonlinearity creates no bad local minima in resnets.
\newblock {\em Neural Networks}, 118:167--174, 2019.

\bibitem{kawaguchi2019effect}
Kenji Kawaguchi, Jiaoyang Huang, and Leslie~Pack Kaelbling.
\newblock Effect of depth and width on local minima in deep learning.
\newblock {\em Neural computation}, 31(7):1462--1498, 2019.

\bibitem{gamarnik2019stationary}
David Gamarnik, Eren~C K{\i}z{\i}lda{\u{g}}, and Ilias Zadik.
\newblock Stationary points of shallow neural networks with quadratic
  activation function.
\newblock {\em arXiv preprint arXiv:1912.01599}, 2019.

\bibitem{lecun2010mnist}
Yann LeCun, Corinna Cortes, and CJ~Burges.
\newblock Mnist handwritten digit database. at\&t labs, 2010.

\bibitem{paszke2017automatic}
Adam Paszke, Sam Gross, Soumith Chintala, Gregory Chanan, Edward Yang, Zachary
  DeVito, Zeming Lin, Alban Desmaison, Luca Antiga, and Adam Lerer.
\newblock Automatic differentiation in pytorch.
\newblock 2017.

\bibitem{kingma2014adam}
Diederik~P Kingma and Jimmy Ba.
\newblock Adam: A method for stochastic optimization.
\newblock {\em arXiv preprint arXiv:1412.6980}, 2014.

\bibitem{demmel2007fast}
James Demmel, Ioana Dumitriu, and Olga Holtz.
\newblock Fast linear algebra is stable.
\newblock {\em Numerische Mathematik}, 108(1):59--91, 2007.

\end{thebibliography}
